\documentclass[10pt,twocolumn,letterpaper]{article}

\usepackage{iccv}              

\usepackage{graphicx}
\usepackage{booktabs}

\usepackage{amsthm}
\usepackage{amsmath,amsfonts,bm}

\newtheorem{definition}{Definition}

\usepackage{thm-restate}
\usepackage[linesnumbered,ruled,vlined]{algorithm2e}

\usepackage{subcaption}
\usepackage{caption}

\usepackage{multirow}
\usepackage{colortbl}

\usepackage{enumitem}

\definecolor{iccvblue}{rgb}{0.21,0.49,0.74}
\usepackage[pagebackref,breaklinks,colorlinks,allcolors=iccvblue]{hyperref}

\usepackage{amsmath,amsfonts,bm}

\def\eqref#1{equation~\ref{#1}}

\def\1{\bm{1}}

\DeclareMathAlphabet{\mathsfit}{\encodingdefault}{\sfdefault}{m}{sl}
\SetMathAlphabet{\mathsfit}{bold}{\encodingdefault}{\sfdefault}{bx}{n}

\DeclareMathOperator*{\argmin}{arg\,min}

\def\paperID{9747} 
\def\confName{ICCV}
\def\confYear{2025}

\title{Resolving Token-Space Gradient Conflicts: Token Space Manipulation for Transformer-Based Multi-Task Learning}

\author{Wooseong Jeong\\
KAIST\\
{\tt\small stk14570@kaist.ac.kr}
\and
Kuk-Jin Yoon\\
KAIST\\
{\tt\small kjyoon@kaist.ac.kr}\\
}

\begin{document}
\maketitle

\renewcommand{\thefootnote}{\fnsymbol{footnote}}
\footnotetext[1]{Our source code is available at: \url{https://github.com/wooseong97/DTME-MTL}}

\begin{abstract}
Multi-Task Learning (MTL) enables multiple tasks to be learned within a shared network, but differences in objectives across tasks can cause negative transfer, where the learning of one task degrades another task's performance. While pre-trained transformers significantly improve MTL performance, their fixed network capacity and rigid structure limit adaptability. Previous dynamic network architectures attempt to address this but are inefficient as they directly convert shared parameters into task-specific ones. We propose Dynamic Token Modulation and Expansion (DTME-MTL), a framework applicable to any transformer-based MTL architecture. DTME-MTL enhances adaptability and reduces overfitting by identifying gradient conflicts in token space and applying adaptive solutions based on conflict type. Unlike prior methods that mitigate negative transfer by duplicating network parameters, DTME-MTL operates entirely in token space, enabling efficient adaptation without excessive parameter growth. Extensive experiments demonstrate that DTME-MTL consistently improves multi-task performance with minimal computational overhead, offering a scalable and effective solution for enhancing transformer-based MTL models.
\end{abstract}    
\section{Introduction}
Multi-Task Learning (MTL) enables multiple tasks to be learned simultaneously within a shared network, improving generalization, efficiency, and convergence speed compared to training separate models \citep{caruana1997multitask}. However, conflicting objectives among tasks can lead to \textit{negative transfer}, where learning one task degrades the performance of another \citep{crawshaw2020multi}. The key challenge lies in designing architectures that effectively balance shared and task-specific representations to mitigate negative transfer.

Transformer-based MTL architectures \citep{xu2023multi, xu2023demt, ye2022taskprompter} leverage the strong generalization capabilities of large-scale pre-trained networks such as Vision Transformers (ViTs) \citep{vit}. By utilizing pre-trained transformers trained on large open-source datasets, these architectures demonstrate improved generalization compared to conventional CNN-based MTL methods \citep{RN9, RN29, RN32, pap, RN51, RN52, RN49, RN50}. However, they typically rely on predefined modules, such as Task Prompter mechanisms \citep{xu2023multi, xu2023demt, ye2022taskprompter} and Mixture of Experts (MoE) \citep{riquelme2021scaling, zhang2022mixture, fan2022m3vit, mustafa2022multimodal, chen2023mod}, to separate shared and task-specific components. These rigid structures struggle to adapt to dynamic task relationships, leading to inefficient information sharing and suboptimal performance. The degree of task specialization required varies across different network depths \citep{dwivedi2019representation}: high-level tasks such as semantic segmentation demand greater capacity in deeper layers, while low-level tasks like surface normal estimation rely more on shallower layers. Ideally, MTL architectures should dynamically adjust the allocation of shared and task-specific representations to accommodate these variations. However, existing transformer-based MTL frameworks are inherently constrained by their fixed network capacity, limiting their ability to adapt to evolving task dependencies and effectively mitigate negative transfer.

A straightforward approach to addressing these limitations is to increase the size of the transformer backbone. While this expands network capacity, it has a major drawback: it prevents the use of open-source pre-trained networks, which provide strong initialization and generalization capabilities across multiple tasks. Training a larger network from scratch requires massive computational resources and large-scale datasets, making this approach impractical for many applications. Instead, an effective MTL framework should refine existing architectures to retain the advantages of pre-trained transformers while improving adaptability to task-specific needs.

To achieve this, efficient adaptation methods for pre-trained transformer-based MTL architectures are needed. Unlike approaches that build MTL frameworks from scratch or rely on task optimization within a fixed network capacity, we focus on adapting and enhancing predefined architectures while preserving their core design. This allows existing MTL models to be improved efficiently while dynamically adjusting task-specific representations. Despite its potential, the challenge of how to adaptively expand existing multi-task networks remains an underexplored problem.

One possible approach for adapting models during fine-tuning is the use of multi-task optimization techniques \citep{RN23, RN26, RN25, liu2024famo, RN19, RN36, RN20, RN18, liu2021towards, navon2022multi, senushkin2023independent}, which mitigate negative transfer by adjusting task loss weights or modifying gradients. While these methods help balance task performance, they remain limited by a fixed network design and cannot expand model capacity.

A more direct approach involves dynamic network architectures, such as Recon \citep{guangyuan2022recon}, which directly expand network capacity to mitigate negative transfer. Recon measures conflicting gradients \citep{RN20} in each layer—where gradients from different tasks point in opposing directions—and transforms the most conflicting layers into task-specific ones. While this increases flexibility by expanding the capacity of predefined architectures, directly converting shared parameters into task-specific ones in transformers leads to parameter inefficiency, excessive computational overhead, and a higher risk of overfitting. Consequently, its scalability to large transformer-based architectures is limited. 

To address these challenges, we propose \textbf{Dynamic Token Modulation and Expansion (DTME-MTL)}, a novel framework designed to improve pre-trained transformer-based MTL architectures. Unlike previous methods that directly manipulate network parameters, our approach mitigates negative transfer by restructuring the \textit{token space} of multi-task networks. We treat transformer tokens as learnable parameters and analyze their structure using \textit{singular value decomposition} (SVD) to identify gradient conflicts in token space. These conflicts are categorized into two types: \textit{range space conflicts}, which are addressed through modulation via affine transformation of existing tokens, and \textit{null space conflicts}, which are resolved by introducing new task-specific tokens through expansion.

In our experiments, we demonstrate that DTME-MTL effectively enhances multi-task performance with minimal parameter overhead. Additionally, our results highlight that resolving task conflicts in the token space improves adaptability while mitigating overfitting.

Our main contributions are summarized as follows:
\begin{itemize}[leftmargin=*]
\item We propose DTME-MTL which dynamically modulates and expands token spaces to mitigate negative transfer in transformer-based multi-task architectures.
\item We introduce a structured approach to resolving gradient conflicts in token space by categorizing them into range and null space conflicts, demonstrating how this improves multi-task performance.
\item DTME-MTL is an \textit{off-the-shelf} solution that seamlessly integrates with existing state-of-the-art transformer-based MTL architectures, enhancing performance with minimal computational overhead.
\end{itemize}
\section{Related Works}
\noindent\textbf{Multi-Task Learning in Vision Transformers.} Originally designed for NLP tasks, transformers have outperformed existing CNN models in various computer vision tasks. Attempts have been made to incorporate Vision Transformer \citep{vit, swin, pvt, focal, segformer, crossformer} in MTL. MTFormer \citep{mtformer} employs a shared transformer encoder and decoder with a cross-task attention mechanism. MulT \citep{mult} utilizes a shared attention mechanism to model task dependencies based on the Swin transformer. InvPT \citep{invpt} focuses on global spatial position and multi-task context for dense prediction tasks through multi-scale feature aggregation. Mixture of Experts (MoE), inspired by the NLP domain, divides the model into predefined expert groups, adaptively shared or devoted to specific tasks during the learning phase \citep{riquelme2021scaling, zhang2022mixture, fan2022m3vit, mustafa2022multimodal, chen2023mod, huang2024going}. Task prompter \citep{xu2023multi, xu2023demt, ye2022taskprompter} uses task-specific tokens to encapsulate task-specific information and employs cross-task interactions to enhance multi-task performance.

\vspace{2pt}
\noindent\textbf{Dynamic Network Architectures for MTL.}  Dynamic networks adapt their structure during training to improve efficiency and task specialization. Several methods have explored dynamic architectures for MTL. Channel-wise dynamic allocation \citep{bragman2019stochastic} assigns different convolutional channels to different tasks, but this method is not directly applicable to transformer-based architectures. Neural Architecture Search (NAS) for MTL \citep{liang2018evolutionary, pasunuru2019continual, guo2020learning, gao2020mtl, bruggemann2020automated, silvestri2020one, cai2021multi} explores optimal network configurations but is computationally expensive and incompatible with large pre-trained backbone models such as ViTs \citep{vit}. Recon \citep{guangyuan2022recon} transforms shared parameters directly into task-specific ones to handle conflicting gradients. Unlike most dynamic network architectures, our approach focuses on a dynamic system that can be directly applied to transformer-based multi-task architectures, leveraging pre-trained backbones while maintaining a reasonable computational cost.

\begin{figure*}[t]
    \includegraphics[width=0.95\linewidth]{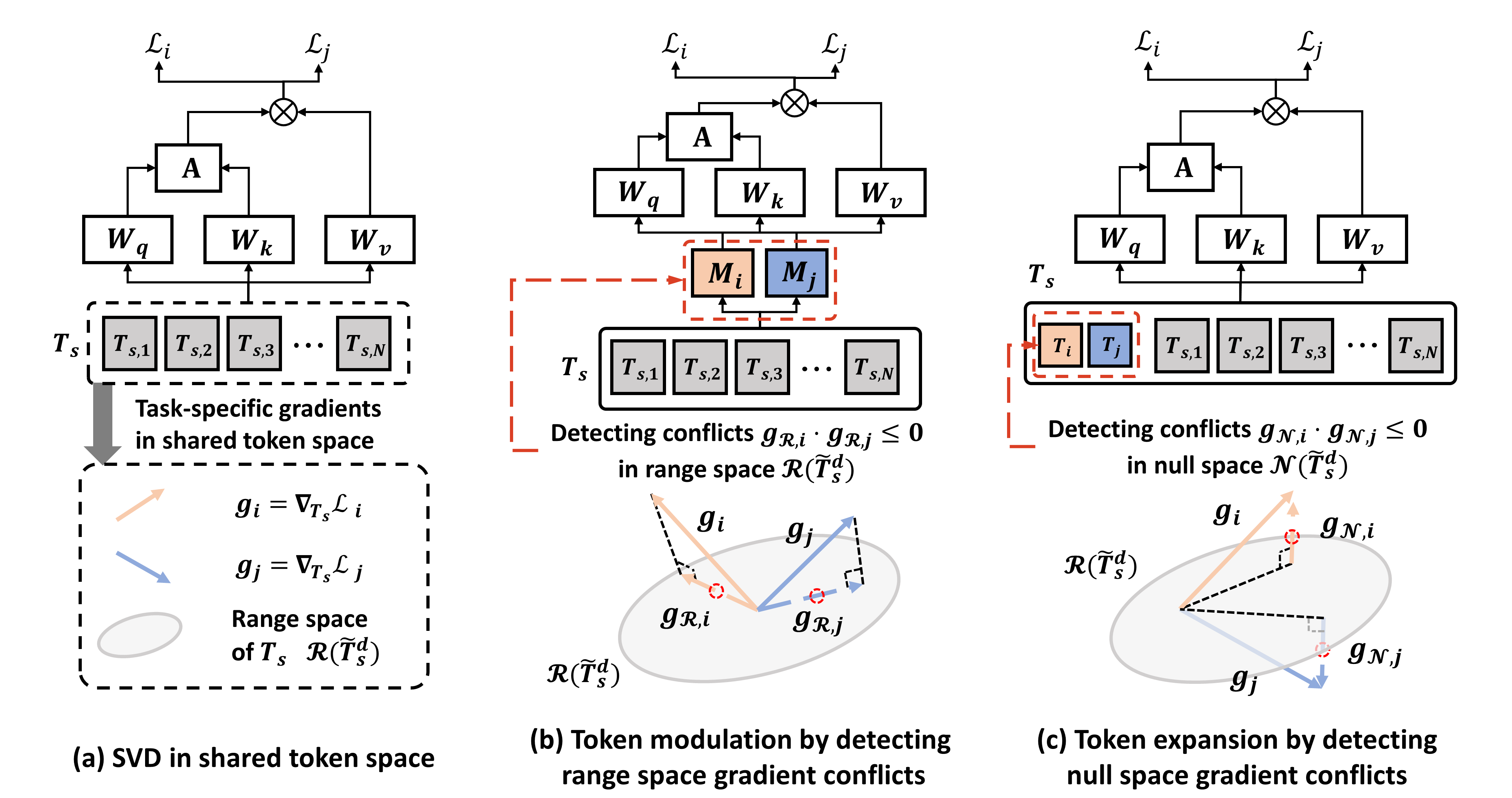}
    \caption{Framework overview of the proposed DTME-MTL. (a) At each network layer, we compute the input token's range space $\mathcal{R}(\widetilde{\mathcal{T}}_s^{d})$ and their task-specific gradients, determining principal vectors from the uncentered covariance of $\mathcal{T}_s$. (b) In cases where task-specific gradients conflict in the range space of $\widetilde{\mathcal{T}}_s^{d}$ (e.g. $g_{\mathcal{R},i} \cdot g_{\mathcal{R},j} \leq 0$), modulation is applied to $\mathcal{T}_s$ by introducing $\mathcal{M}_i$ and $\mathcal{M}_j$. (c) When task-specific gradients conflict within the null space of $\widetilde{\mathcal{T}}_s^{d}$ (e.g. $g_{\mathcal{N},i} \cdot g_{\mathcal{N},j} \leq 0$), 
    task-specific tokens $\mathcal{T}_i$ and $\mathcal{T}_j$ are added.}
    \label{fig:overview}
    \vspace{-2pt}
\end{figure*}

\vspace{2pt}
\noindent\textbf{Multi-Task Optimization.} Optimizing the MTL aims to address negative transfer by adjusting the relative weighting of task losses or directly manipulating gradients. Task-dependent uncertainty \citep{RN23} is utilized to weigh the loss of multiple tasks. \citet{RN26} considers the rate of loss descent, while \citep{RN25} prioritizes tasks based on difficulty. Recently, \citet{liu2024famo} proposed updating task weights based on the loss history. In contrast, approaches like \citep{RN19, RN36, RN20, RN18, liu2021towards, navon2022multi, senushkin2023independent} directly modify task gradients to achieve the desired balance. PCGrad \citep{RN20} analyzes negative transfer by identifying conflicting gradients in the shared parameters of the network. \citet{jiang2024forkmerge} suggests a positive link between negative transfer and conflicting gradients in auxiliary task learning. However, the conventional view in MTL considers conflicting gradients a key factor contributing to negative transfer in joint multi-task learning optimization \citep{RN19, RN36, RN20, RN18, liu2021towards, navon2022multi, senushkin2023independent, jeong2024quantifying, jeong2025selective}, where tasks are learned together rather than serving as auxiliary tasks. Therefore, we adopt a similar perspective. Normalized gradients are employed to prevent spillover between tasks \citep{RN24}, whereas \citet{RN21} introduce stochasticity to the network's parameters based on the consistency in the sign of gradients. RotoGrad \citep{RN22} rotates the feature space of the network to narrow the gap between tasks.

\section{Preliminaries}
\label{preliminaries}
In multi-task learning, the network learns a set of tasks $\{\tau_i\}_{i=1}^{\mathcal{K}}$ jointly, where $\mathcal{K}$ is the number of tasks. Each task $\tau_i$ has its own loss function $\mathcal{L}_{i}$. The network parameter $\Theta$ can be classified into $\Theta = \{\Theta_{s}, \Theta_{1} ,...,\Theta_{\mathcal{K}}\}$ where $\Theta_{s}$ is shared parameter across all tasks and $\Theta_i$ is task-specific parameters devoted to task $\tau_i$. Then, the objective function of multi-task learning is to minimize the weighted sum of all tasks' losses: $\Theta^* = \argmin_{\Theta}\sum^{\mathcal{K}}_{i=1} w_i\mathcal{L}_i(\Theta_s, \Theta_i)$ where $w_i$ represents the scale of the task-specific loss $\mathcal{L}_i$. A phenomenon called conflicting gradients \cite{RN20}, where the gradients of each objective point in different directions, has been identified as a main cause of negative transfer.
\begin{definition}[Conflicting gradients]
Define $g_i$ as the gradient of task $\tau_i$ with respect to the shared parameters $\Theta_s$ as $g_i = \nabla_{\Theta_s}\mathcal{L}_i(\Theta_s, \Theta_i)$. Let $g_i$ and $g_j$ represent the gradients for a pair of tasks $\tau_i$ and $\tau_j$ where $i \neq j$. If $g_i \cdot g_j \leq 0$, these two gradients are termed conflicting gradients.
\label{def:conflicting_gradients}
\end{definition}
However, the role of conflicting gradients remains a topic of debate. While conventional MTL optimization studies \cite{RN19, RN36, RN20, RN18, liu2021towards, navon2022multi, senushkin2023independent, jeong2024quantifying} consider conflicting gradients as a main cause of negative transfer, \citet{jiang2024forkmerge} argue that they can serve as a form of regularization that improves generalization when present in network parameters. Our findings align with \citet{jiang2024forkmerge} in that directly resolving conflicting gradients by converting shared parameters into task-specific ones \cite{guangyuan2022recon} leads to overfitting when applied to transformers. In contrast, we propose a token-based network expansion approach that categorizes gradient conflicts within token space and adapts accordingly, mitigating negative transfer while maintaining generalization.
\begin{figure*}[t]
\centering
    \includegraphics[width=0.85\linewidth]{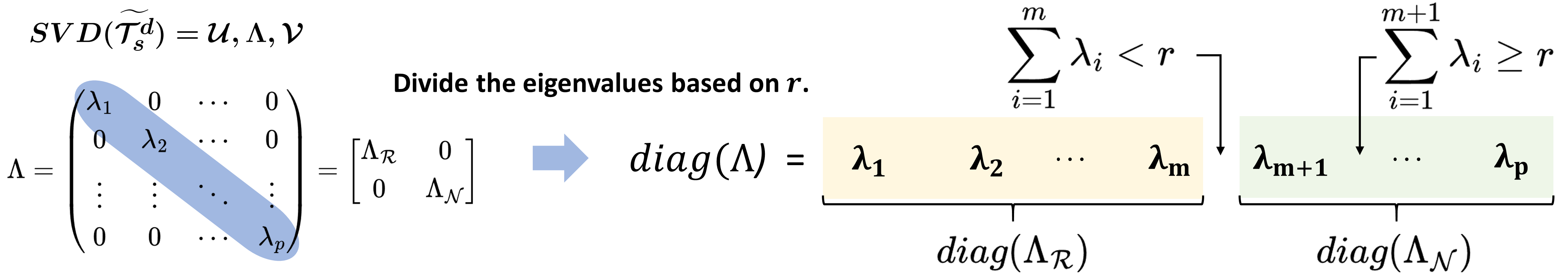}
    \caption{The process approximates the range and null spaces of $\Tilde{\mathcal{T}}_s^d$ based on the proportion of total variance, $r$. These eigenvalues are arranged in descending order, satisfying $\lambda_i \geq \lambda_j$ if $i < j$. If $r$ is greater than the sum up to $\lambda_m$ and smaller than the sum up to $\lambda_{m+1}$, then we select the set $\{\lambda_i\}_{i=1}^{m}$ as $\Lambda_{\mathcal{R}}$, and the remaining set $\{\lambda_i\}_{i=m+1}^{p}$ as $\Lambda_{\mathcal{N}}$.}
\label{fig:variance_proportion}
\vspace{-2pt}
\end{figure*}

\section{Method}
In order to mitigate negative transfer by ensuring sufficient space for tasks, we adopt token-based network expansion. Initially, we define the token space as the output of each layer in the transformer block through singular value decomposition (SVD). Subsequently, we categorize conflicts in task-specific gradients into two types: conflicts in the range space of tokens and conflicts in the null space of tokens. Finally, based on the type of conflict, we introduce efficient token modulation and expansion techniques for transformer-based multi-task architectures.

\subsection{Defining Token Space using SVD}
\label{subsec:token_space}
In this section, we create a vector space consisting of shared tokens in a transformer, aiming to classify the types of conflicting gradients. More specifically, we approximate the range space and null space of the uncentered covariance of the tokens before applying our methods.

Let's consider a dataset $\{\mathcal{X}_l, \mathcal{Y}_l\}_{l=1}^{n}$, where $\mathcal{X}_l$ represents the input, $\mathcal{Y}_l$ denotes the label, and $n$ is the number of samples. Denote input shared token for a layer $d$ as $\mathcal{T}_s^{l,d} = [\mathcal{T}_{s,1}^{l,d}, \mathcal{T}_{s,2}^{l,d}, ..., \mathcal{T}_{s,N}^{l,d}]$ where $N$ is the total number of shared tokens in that layer. Every token $\mathcal{T}_{s,k}^{l,d} \in \mathbb{R}^{p}$ represents the output of the transformer layer $d-1$ for the corresponding input data $\mathcal{X}_l$, where $p$ is the hidden dimension of the token embedding. Let's consider a total of $D$ transformer layers. Next, the uncentered covariance of the token in layer $d$ (where $1\leq d\leq D$) is as follows:

\begin{align}
    \widetilde{\mathcal{T}}_s^{d} = \frac{1}{n}\sum_{l=1}^{n}(\mathcal{T}_s^{l,d})(\mathcal{T}_s^{l,d})^T
\end{align}
$\widetilde{\mathcal{T}}_s^d$ is a square matrix of dimensions $p \times p$. To define the token space, we apply SVD to $\widetilde{\mathcal{T}}_s^{d}$. Following this, we can divide vector space formed by $\widetilde{\mathcal{T}}_s^{d}$ into its range space $\mathcal{R}(\widetilde{\mathcal{T}}_s^{d})$ and null space $\mathcal{N}(\widetilde{\mathcal{T}}_s^{d})$ depending on the magnitude of eigenvalue $\Lambda$. The process is illustrated below:
\begin{align}
    \widetilde{\mathcal{T}}_s^d = \mathcal{U} \Lambda \mathcal{V}^T, \hspace{10pt}
     \Lambda =
        \begin{bmatrix}
        \Lambda_{\mathcal{R}} & 0 \\
        0 & \Lambda_{\mathcal{N}} 
        \end{bmatrix}
\label{eq:define_space}
\end{align}
where $\Lambda$ is a diagonal matrix. Each $\Lambda_{\mathcal{R}}$ and $\Lambda_{\mathcal{N}}$ represent submatrices of $\Lambda$ containing the eigenvalues of the range space and null space, respectively. Both $\mathcal{U}$ and $\mathcal{V}$ are square matrices, each with dimensions $p \times p$.

From \cref{eq:define_space}, we obtain a mathematical tool to define the range and null space of the covariance of the token, $\widetilde{\mathcal{T}}_s^d$. To approximate the range space, we choose the eigenvalue $\Lambda_{\mathcal{R}}$ along with their corresponding eigenvectors from $\mathcal{U}_{\mathcal{R}}$. On the other hand, when approximating the null space, we should select the eigenvalues $\Lambda_{\mathcal{N}}$ and their corresponding eigenvectors from $\mathcal{U}_{\mathcal{N}}$. Ideally, we should choose eigenvalues that are exactly zero to form the null space. However, in practice, $\Lambda$ can not be precisely zero. Therefore, it's essential to establish a criterion for selecting the eigenvalue to distinguish between these two spaces.

Instead of introducing a new manually designed rule for approximating each range and null space of $\widetilde{\mathcal{T}}_s^d$, we opt to directly employ the evaluation tool for the SVD \citep{jollife2016principal} as criteria for determining the range and null space of tokens. In assessing the accuracy of the SVD approximation, the proportion of total variance, denoted as $r$, has been employed:
\begin{align}
    r = \frac{\sum_{\lambda \in diag({\Lambda_\mathcal{N}})} \lambda}{\sum_{\lambda \in diag({\Lambda_\mathcal{R}})} \lambda}
\label{eq:r}
\end{align}
The $diag$ function serves as an operator, returning a set containing the diagonal entries of the input matrix. In our approach, we employ \cref{eq:r} to directly divide the range and null space of $\widetilde{\mathcal{T}}_s^d$. As depicted in \cref{fig:variance_proportion}, the diagonal elements of the matrix $\Lambda$, obtained through the SVD of $\Tilde{\mathcal{T}}_s^d$, are arranged in descending order based on their magnitudes. We can select the index of the eigenvalue $m$ such that the sum of eigenvalues up to order $m$ is smaller than $r$, and the sum up to $m+1$ is larger than $r$. This index serves as a boundary to divide the range space and null space of $\Tilde{\mathcal{T}}_s^d$.

\subsection{Types of Gradient Conflicts}
\label{sec:type_conflict}
In Section \ref{subsec:token_space}, we create a $p$-dimensional vector space using the uncentered covariance of the shared token $\widetilde{\mathcal{T}}_s^d$, linked to the input data set $\{\mathcal{X}\}_{l=1}^{n}$. This vector space is divided into the range and null space, with each space spanned by eigenvectors corresponding to singular values selected based on a specified ratio $r$. In the upcoming sections, we pinpoint the types of gradient conflict within the vector space we've constructed. We then address these conflicts adaptively by introducing token modulation and expansion techniques.

Using \cref{eq:define_space} and \cref{eq:r}, we can partition the eigenvectors of the $p$-dimensional vector space into its range and null space. Now, let's consider the shared tokens $\mathcal{T}_s^{l,d} = [\mathcal{T}_{s,1}^{l,d}, \dots, \mathcal{T}_{s,N}^{l,d}]$, omitting the explicit notation of $l,d$ for simplicity. For example, we write $\mathcal{T}_s^{l,d} \rightarrow \mathcal{T}_s$, $\mathcal{T}_{s,k}^{l,d} \rightarrow \mathcal{T}_{s,k}$, and $\Tilde{\mathcal{T}}_s^d \rightarrow \Tilde{\mathcal{T}}_s$.
We treat $\mathcal{T}_s$ as network parameters, for which gradients can be computed during the backpropagation process. Then, for each loss $\mathcal{L}_i$, the task-specific gradient for $\mathcal{T}_{s,k}$ is denoted as $g_i = \nabla_{\mathcal{T}_{s,k}} \mathcal{L}_i$. Consequently, we obtain task-specific gradients $\{g_i\}_{i=1}^{\mathcal{K}}$ corresponding to a set of losses $\{\mathcal{L}_i\}_{i=1}^{\mathcal{K}}$ for $\mathcal{T}_s$ as shown in \cref{fig:overview}-(a).

Each task-specific gradient $g_i$ can be decomposed into two components, $g_{\mathcal{R},i}$ and $g_{\mathcal{N},i}$, through projection onto the range and null space of $\Tilde{\mathcal{T}}_s^d$, respectively. This breakdown is expressed as follows:
\begin{align}
& g_{\mathcal{R},i} = (\mathcal{U}_\mathcal{R} \mathcal{U}_\mathcal{R}^T)\nabla_{\mathcal{T}_{s,k}}\mathcal{L}_i
& g_{\mathcal{N},i} = (\mathcal{U}_\mathcal{N} \mathcal{U}_\mathcal{N}^T)\nabla_{\mathcal{T}_{s,k}}\mathcal{L}_i
\label{eq:projection}
\end{align}
$\mathcal{U}_\mathcal{R}$ and $\mathcal{U}_\mathcal{N}$ are orthogonal matrices that consist of eigenvectors of the range space and null space, respectively, with each column representing one eigenvector. Then, 
the matrices $ (\mathcal{U}_\mathcal{R} \mathcal{U}_\mathcal{R}^T) $ and $ (\mathcal{U}_\mathcal{N} \mathcal{U}_\mathcal{N}^T) $ function as projection operators onto the range and null spaces, respectively.

Building upon the concept of conflicting gradients outlined in \cref{def:conflicting_gradients}, we classify conflicts into two types based on the space in which they occur: range space conflicts and null space conflicts. Specifically, conflicts in the range space of tokens occur when $g_{\mathcal{R},i} \cdot g_{\mathcal{R},j} \leq 0$ for any pair of $i$ and $j$ where $i \neq j$. Likewise, conflicts in the null space of tokens emerge when $g_{\mathcal{N},i} \cdot g_{\mathcal{N},j} \leq 0$.

\subsection{Token Modulation and Expansion}
Building on the gradient conflict types defined in \cref{sec:type_conflict}, we propose adaptive strategies to mitigate task interference. Specifically, if gradient conflicts occur in the range space, we apply an affine transformation to modulate tokens, while conflicts in the null space are addressed by introducing additional tokens to expand the feature space.
This distinction is particularly relevant in the transfer learning setting, where a pretrained transformer backbone is used, and task interference must be handled during fine-tuning. According to \cite{neyshabur2020being}, training from pretrained weights constrains the model within the same basin of the loss landscape, preserving a feature space similar to that of the pretrained network. This insight guides our separation of token space into range and null spaces: conflicts in the row space indicate that the network already has relevant interpretative capabilities and can be resolved through rotation or scaling, whereas conflicts in the null space suggest the need for additional features, requiring token expansion to enhance the model’s capacity.

\vspace{2pt}\noindent
\textbf{Token Modulation.} In situations where task-specific gradients conflict within the range space of $\widetilde{\mathcal{T}}_s$, such as $g_{\mathcal{R},i} \cdot g_{\mathcal{R},j} \leq 0$, modulators $\mathcal{M}_i$ and $\mathcal{M}_j$ are added after the shared token $\mathcal{T}_s$ as shown in \cref{fig:overview}-(b). The token modulator $\mathcal{M}$ is a straightforward affine transformation that modulates the shared token $\mathcal{T}_s$ along the channel dimension. To elaborate, considering the embedding dimension of the transformer $p$ and the number of shared tokens is $N$, we can arrange $\mathcal{T}_s$ in the form $[\mathcal{T}_{s,1},\dots,\mathcal{T}_{s,N}]$. This arrangement turns $\mathcal{T}_s$ into a $p\times N$ matrix. The modulator $\mathcal{M}$, which incorporates weight and bias $W, b \in \mathbb{R}^{p}$, performs the transformation $W \odot \mathcal{T}_{s,i} + b$, where $\odot$ denotes elementwise multiplication. When the gradient lies in the row space of $\Tilde{\mathcal{T}}_s$, \cref{theorem1} demonstrates that applying token modulation can effectively resolve gradient conflicts, lowering the multi-task loss.
\begin{restatable}[]{proposition}{theomone}
\label{theorem1}
When the input token $\mathcal{T}_{in}$ for input sample $\mathcal{X}_l$ spans the range space of $\Tilde{\mathcal{T}}_s$, optimizing the token modulators $\{\mathcal{M}_i\}_{i=1}^{\mathcal{K}}$ reduces gradient conflicts in the row space of $\Tilde{\mathcal{T}}_s$ and leads to a reduction in the multi-task loss.
\end{restatable}

\begin{algorithm}[t]
\DontPrintSemicolon
\caption{DTME-MTL}\label{alg:alg1}

\KwData{Task $\{\tau_i\}^\mathcal{K}_{i=1}$, Loss function $\{\mathcal{L}_i\}^\mathcal{K}_{i=1}$, \newline
        Dataset $\{\mathcal{X}_l, \mathcal{Y}_l\}_{l=1}^{n}$, Shared tokens $\mathcal{T}_s^{l,d} = \{\mathcal{T}_{s,i}^{l,d}\}_{i=1}^{N}$, Depth of the Network $D$ \newline} 

\For{each layer of the network ($d \gets 1$ to $D$)}{
    Get tokens $\{\mathcal{T}_s^{l,d}\}_{l=1}^n$ for the layer $d$ corresponding to input $\{\mathcal{X}_l\}_{l=1}^n$ \\
    
    Calculate uncentered covariance.
    $\widetilde{\mathcal{T}}_s^d = \frac{1}{n}\sum_{l=1}^{n}(\mathcal{T}_s^{l,d})(\mathcal{T}_s^{l,d})^T$ \\

    Singular value decomposition.
    $\mathcal{U}, \Lambda, \mathcal{V} = SVD(\widetilde{\mathcal{T}}_s^d)$ \\

    Divide range and null space.
    $\mathcal{U} = [\mathcal{U}_{\mathcal{R}}, \mathcal{U}_{\mathcal{N}}]$ \\
    Projection to range space.
    $\{g_{\mathcal{R},i}\}_{i=1}^{\mathcal{K}} = \{(\mathcal{U}_{\mathcal{R}} \mathcal{U}_{\mathcal{R}}^T)\nabla_{\mathcal{T}_{s,k}^{l,d}}\mathcal{L}_i\}_{i=1}^{\mathcal{K}}$\\

    Projection to null space.
    $\{g_{\mathcal{N},i}\}_{i=1}^{\mathcal{K}} = \{(\mathcal{U}_{\mathcal{N}} \mathcal{U}_{\mathcal{N}}^T)\nabla_{\mathcal{T}_{s,k}^{l,d}}\mathcal{L}_i\}_{i=1}^{\mathcal{K}}$\\
    
    \If{$g_{\mathcal{R},i} \cdot g_{\mathcal{R},j} \leq 0$}{Insert token modulators $\mathcal{M}_i$ and $\mathcal{M}_j$}
    \If{$g_{\mathcal{N},i} \cdot g_{\mathcal{N},j} \leq 0$}{Insert task-specific tokens $\mathcal{T}_i$ and $\mathcal{T}_j$}}
\end{algorithm}

\begin{table*}[t]
\vspace{-10pt}
\caption{We conduct an ablation study on dynamic token modulation and expansion, evaluating the multi-task performance of our method on NYUD-v2 and PASCAL-Context. The results of TE, TM, and their combination, TE+TM are presented. We employ a shared encoder and multiple decoders, using ViT-T as the backbone network.}
\vspace{-5pt}
\centering
\renewcommand\arraystretch{1.00}
\resizebox{0.90\textwidth}{!}{
\tiny
\begin{tabular}{l|cccc|ccccc}
\hline
\multicolumn{1}{c|}{}  & \multicolumn{4}{c|}{NYUD-v2}  & \multicolumn{5}{c}{PASCAL-Context}  \\ \cline{2-10}  
\multicolumn{1}{c|}{}  & Semseg  & Depth  & Normal  & Edge  & Semseg  & Parsing  & Saliency  & Normal  & Edge  \\
\multicolumn{1}{c|}{\multirow{-3}{*}{Model}} & mIoU $\uparrow$  & RMSE $\downarrow$  & mErr $\downarrow$  & odsF $\uparrow$  & mIoU $\uparrow$  & mIoU $\uparrow$  & maxF $\uparrow$  & mErr $\downarrow$  & odsF $\uparrow$  \\ \hline
Baseline (ST)  &39.35  &0.6611  &22.14  &59.68  &67.96  &58.90  &83.76  &15.65  &47.70  \\ \hline
Baseline (MT)  &34.13  &0.6732  &22.51  &55.30  &54.47  &51.48  &82.04  &16.22  &41.28  \\
TM             &37.85  &0.6490  &21.75  &56.92  &64.28  &55.10  &83.02  &15.40  &45.80  \\
TE             &37.25  &0.6553  &21.87  &57.00  &60.51  &54.00  &82.85  &15.55  &44.98  \\
TM+TE          &38.27  &0.6370  &21.64  &57.90  &66.18  &56.29  &83.41  &15.26  &47.00  \\
Gain (vs. MT)  &\textbf{$\triangle$4.14}  &\textbf{$\triangle$0.0362}  &\textbf{$\triangle$0.87}  &\textbf{$\triangle$2.60}  &\textbf{$\triangle$11.71}  &\textbf{$\triangle$4.81}  &\textbf{$\triangle$1.37}  &\textbf{$\triangle$0.96}  &\textbf{$\triangle$5.72}  \\ \hline
$\triangle_m$ $\uparrow$    & \multicolumn{4}{c|}{\textbf{0.044}} & \multicolumn{5}{c}{-\textbf{1.289}}  \\ \hline
\#Param $\uparrow$ (\%) & \multicolumn{4}{c|}{\textbf{0.24}} & \multicolumn{5}{c}{\textbf{0.30}}  \\ \hline
\end{tabular}}
\label{tab:tab_ablation}
\end{table*}
\begin{figure*}[t]
    \vspace{-5pt}
    \centering
    \includegraphics[width=0.90\linewidth]{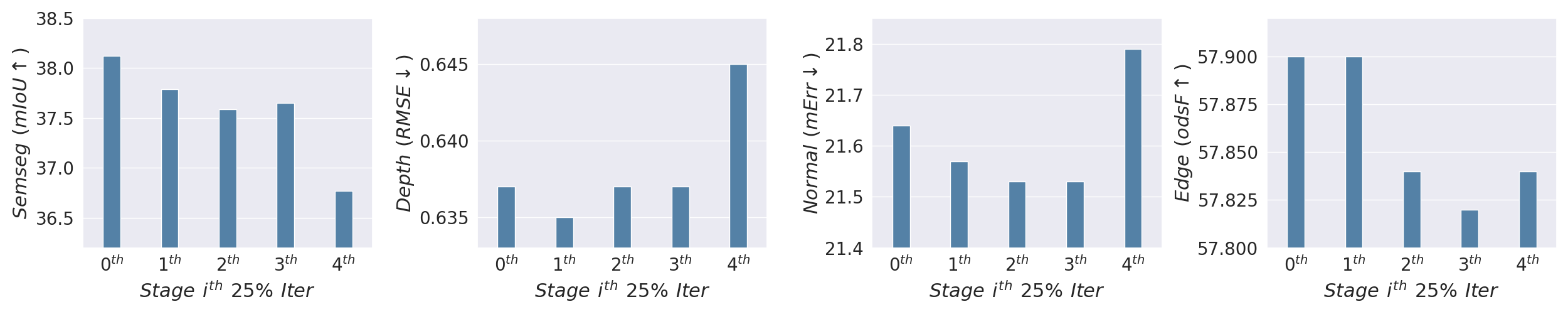}
    \vspace{-8pt}
    \caption{Task performance varies based on when we expand the network. To determine the optimal timing, we assess expansions at the beginning of training and at the end of each quarter iteration, monitoring the corresponding changes in performance.}
\label{fig:iter_perform}
\vspace{-5pt}
\end{figure*}

\vspace{2pt}\noindent
\textbf{Token Expansion.} Similarly, in cases where task-specific gradients conflict within the null space of $\widetilde{\mathcal{T}}_s$, such as $g_{\mathcal{N},i} \cdot g_{\mathcal{N},j} \leq 0$, task-specific tokens $\mathcal{T}_i$ and $\mathcal{T}_j$ are added alongside shared tokens $\mathcal{T}_s$ as shown in \cref{fig:overview}-(c). The task-specific tokens $\{\mathcal{T}_i\}_{i=1}^{\mathcal{K}}$ are concatenated with shared tokens before entering the transformer block. Consequently, each task-specific token acquires task-specific information within that layer. Specifically, in a standard transformer block, self-attention is performed for each pair of tokens in the form of $[\mathcal{T}_{s,1},\dots,\mathcal{T}_{s,N}] \times [\mathcal{T}_{s,1},\dots,\mathcal{T}_{s,N}]$. With token expansion, attention is extended to include $[\mathcal{T}_{s,1},\dots,\mathcal{T}_{s,N}] \times [\mathcal{T}_{1},\dots,\mathcal{T}_{\mathcal{K}}]$ on the output.
\cref{theorem2} explains how expanding the token space to address gradient conflicts in the null space of $\Tilde{\mathcal{T}}_s$ leads to a reduction in multi-task loss when the gradient lies in this null space. All proofs are provided in Supple~\ref{Append:proof}.
\begin{restatable}[]{proposition}{theomtwo}
\label{theorem2}
When the input token $\mathcal{T}_{in}$ for input sample $\mathcal{X}_l$ spans the null space of $\Tilde{\mathcal{T}_s}$, token expansion using $\{\mathcal{T}_i\}_{i=1}^{\mathcal{K}}$ alleviates the increase in multi-task loss caused by gradient conflicts in the null space of $\Tilde{\mathcal{T}}_s$.
\end{restatable}
The complete procedure for DTME-MTL is outlined in Alg.\ref{alg:alg1}. 
Handling gradient conflicts in token space improves adaptability and reduces overfitting while being more efficient than addressing conflicts at the parameter level.
\begin{table*}[tb]
\vspace{-10pt}
\caption{Performance comparison based on the degree of conflicts in reversed order (Reversed) and randomly selected layers (Random).}
\vspace{-5pt}
\centering
\renewcommand\arraystretch{1.00}
\tiny
\resizebox{0.95\textwidth}{!}{
\begin{tabular}{c|cccc|ccccc}
\hline
\multirow{3}{*}{}  & \multicolumn{4}{c|}{NYUD-v2}   & \multicolumn{5}{c}{PASCAL-Context} \\ \cline{2-10} 
& \multicolumn{1}{c|}{Semseg} & \multicolumn{1}{c|}{Depth}  & \multicolumn{1}{c|}{Normal} & Edge  & \multicolumn{1}{c|}{Semseg} & \multicolumn{1}{c|}{Parsing} & \multicolumn{1}{c|}{Saliency} & \multicolumn{1}{c|}{Normal} & Edge  \\
\multicolumn{1}{c|}{\multirow{-3}{*}{Model}} & \multicolumn{1}{c|}{mIoU $\uparrow$}   & \multicolumn{1}{c|}{RMSE $\downarrow$}   & \multicolumn{1}{c|}{mErr $\downarrow$}   & odsF $\uparrow$ & \multicolumn{1}{c|}{mIoU $\uparrow$}   & \multicolumn{1}{c|}{mIoU $\uparrow$}    & \multicolumn{1}{c|}{maxF $\uparrow$}     & \multicolumn{1}{c|}{mErr $\downarrow$}   & odsF $\uparrow$ \\ \hline
\multicolumn{1}{l|}{TM+TE}     & \multicolumn{1}{c|}{\textbf{38.27}}  & \multicolumn{1}{c|}{\textbf{0.6370}} & \multicolumn{1}{c|}{\textbf{21.64}}  & \textbf{57.90} & \multicolumn{1}{c|}{\textbf{66.18}}  & \multicolumn{1}{c|}{\textbf{56.29}}   & \multicolumn{1}{c|}{\textbf{83.21}}    & \multicolumn{1}{c|}{\textbf{15.26}}  & \textbf{47.00} \\
\multicolumn{1}{l|}{TM+TE (Random)} & \multicolumn{1}{c|}{36.88}  & \multicolumn{1}{c|}{0.6567} & \multicolumn{1}{c|}{22.27}  & 56.30 & \multicolumn{1}{c|}{62.12}  & \multicolumn{1}{c|}{54.43}   & \multicolumn{1}{c|}{82.95}    & \multicolumn{1}{c|}{15.55}  & 45.80 \\
\multicolumn{1}{l|}{TM+TE (Reverse)} & \multicolumn{1}{c|}{34.71}  & \multicolumn{1}{c|}{0.6898} & \multicolumn{1}{c|}{22.59}  & 55.80 & \multicolumn{1}{c|}{57.84}  & \multicolumn{1}{c|}{52.82}   & \multicolumn{1}{c|}{82.75}    & \multicolumn{1}{c|}{15.74}  & 43.20 \\ \hline
\end{tabular}}
\label{tab:ablation_random}
\vspace{-5pt}
\end{table*}
\begin{figure*}[t]
\centering
\vspace{-5pt}
\begin{subfigure}{.46\linewidth}
    \includegraphics[width=\linewidth]{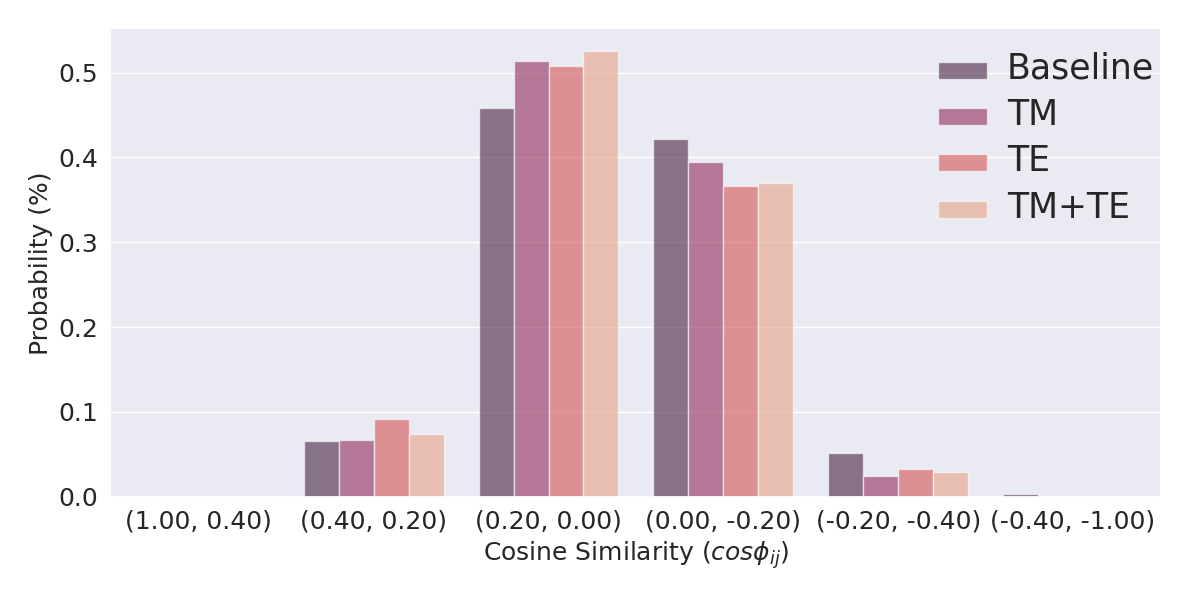}
    \caption{NYUD-v2}
\end{subfigure}
\begin{subfigure}{.46\linewidth}
    \includegraphics[width=\linewidth]{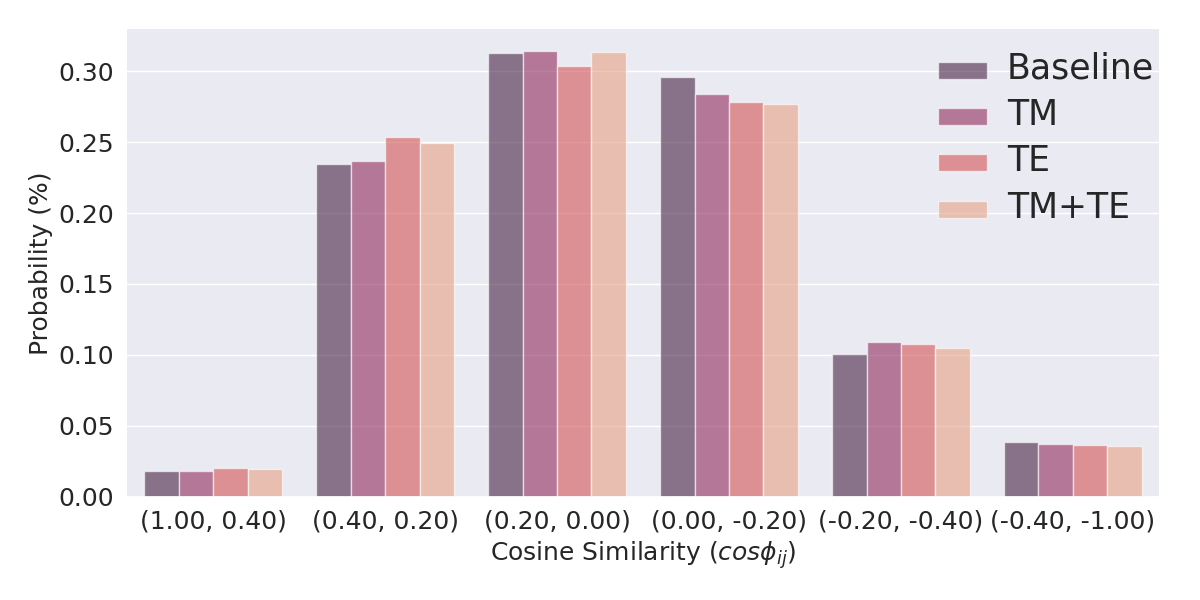}
    \caption{PASCAL-Context}
\end{subfigure}
\vspace{-5pt}
\caption{We evaluate the distribution of gradient conflicts by measuring the cosine similarity between task-specific gradients across all shared parameters throughout the optimization process. This is represented as $cos\phi_{ij}$ in (a) for NYUD-v2 and in (b) for PASCAL-Context.}
\vspace{-4pt}
\label{fig:conflicts}
\end{figure*}

\section{Experiments}
\subsection{Experimental Settings}
\noindent\textbf{Datasets and Evaluation.}
Our method is evaluated on multi-task datasets: NYUD-v2 \citep{RN15}, PASCAL-Context \citep{mottaghi2014role} and Taskonomy \citep{zamir2018taskonomy}. Each of them with 4, 5, 11 tasks. To evaluate the performance of tasks, we employed widely used metrics. To evaluate the multi-task performance, we utilize the metric proposed by \citet{RN2}. It measures the per-task performance $M_{m,i}$ by averaging it with respect to the single-task baseline $M_{b,i}$, as shown in $\triangle_m = (1/\mathcal{K})\sum_{i=1}^{\mathcal{K}}(-1)^{l_i}(M_{m,i}-M_{b,i})/M_{b,i}$. $l_i = 1$ if a lower value of the measure $M_i$ indicates better performance for task $i$, and 0 otherwise.

\noindent\textbf{Baselines and Model Variants.}
For a comprehensive analysis of the proposed DTME-MTL framework, we adopt a typical experimental setup for MTL in our experiments. In \cref{tab:tab_ablation}, `Baseline (MT)' refers to a simple multi-task architecture consisting of a shared transformer backbone and basic task-specific decoders. Each decoder comprises one $3 \times 3$ Conv-BN-ReLU block. `Baseline (ST)' has the same structure as `Baseline (MT)' but is trained with only a single task. We assess the proposed DTME-MTL framework by expanding the network from `Baseline (MT)' and measure the performance gains achieved by the proposed methods. `TM' (Token Modulation) signifies the addition of the proposed token modulator to `baseline (MT)', while `TE' (Token Expansion) indicates the incorporation of task-specific tokens onto `Baseline (MT)'. Finally, `TM+TE' combines both proposed methods. To show how effectively our approach reduces negative transfer, we also compare it with previous multi-task optimization, though our methods can be used alongside them. We include simple gradient descent (GD), GradDrop \citep{RN21}, MGDA \citep{RN36}, PCGrad \citep{RN20}, CAGrad \citep{RN18}, IMTL \citep{liu2021towards}, Nash-MTL \citep{navon2022multi}, and Aligned-MTL \citep{senushkin2023independent}, as well as loss balancing methods such as UW \citep{RN23}, DWA \citep{RN26}, and FAMO \citep{liu2024famo}. We also compare our results with dynamic network architecture such as Recon \citep{guangyuan2022recon}. Further experimental details are summarized in Supple~\ref{Append:exp_settings}.

\begin{table*}[t]
    \centering
    \vspace{-5pt}
    \caption{Comparison of multi-task optimization methods on Taskonomy across 11 tasks. Non-converged results are indicated with a dash.}
    \vspace{-5pt}
    \resizebox{\linewidth}{!}{
    \begin{tabular}{l|ccccccccccc|c}
    \toprule
        Task       & DE     & DZ      & EO     & ET     & Key2D    & Key3D   & N      & PC     & R      & S2D    & S25D   &    \\
        Metric    & L1 Dist. $\downarrow$  & L1 Dist. $\downarrow$ & L1 Dist. $\downarrow$ & L1 Dist. $\downarrow$ & L1 Dist. $\downarrow$ & L1 Dist. $\downarrow$ & L1 Dist. & RMSE $\downarrow$    & L1 Dist. $\downarrow$ & L1 Dist. $\downarrow$ & L1 Dist. $\downarrow$  & $\triangle_m$ $\uparrow$ $(\%)$ \\ \midrule 
        ST    & 0.0199&0.0195&0.1085&0.1714&0.1633&0.0872&0.2715&0.7586&0.1503&0.1742&0.1504& 0.00 \\ \midrule
        GD        & 0.0187&0.0188&0.1301&0.1757&0.1733&0.0942&0.3076&0.7991&0.1826&0.1902&0.1652&- 7.83 \\
        GradDrop \citep{RN21}  & 0.0315&0.0242&0.1390&0.1776&0.1778&0.0976&0.4564&0.8644&0.2088&0.1995&0.1752&- 26.11 \\
        MGDA \citep{RN36}      & - & - & - & - & - & - & - & - & - & - & - & - \\
        UW \citep{RN23}       & 0.0190&0.0190&0.1308&0.1758&0.1734&0.0945&0.3109&0.8009&0.1840&0.1906&0.1657&- 8.43 \\
        DWA \citep{RN26}       & 0.0186&0.0187&0.1294&0.1759&0.1735&0.0938&0.2788&0.7943&0.1805&0.1902&0.1640&- 6.45 \\
        PCGrad \citep{RN20}    & 0.0217&0.0192&0.1298&0.1775&0.1714&0.0939&0.2856&0.7985&0.1817&0.1927&0.1595&- 8.29 \\  
        CAGrad \citep{RN18}    & 0.0219&0.0203&0.1314&0.1800&0.1665&0.0932&0.3039&0.8121&0.1874&0.1953&0.1673&- 10.57 \\
        IMTL \citep{liu2021towards}      & 0.0210&0.0192&0.1282&0.1772&0.1719&0.0936&0.2468&0.7784&0.1734&0.1943&0.1647&- 6.17 \\
        Align-MTL \citep{senushkin2023independent} & 0.0189&0.0193&0.1254&\textbf{0.1728}&\textbf{0.1664}&0.0914&0.3524&0.8640&0.1938&0.1889&0.1582&- 9.41 \\
        Nash-MTL \citep{navon2022multi}  & 0.0201&0.0184&0.1248&0.1764&0.1701&0.0921&0.2658&0.7793&0.1706&0.1914&0.1624&- 5.01 \\
        FAMO \citep{liu2024famo}     & 0.0188&0.0188&0.1300&0.1758&0.1733&0.0942&0.3058&0.7986&0.1826&0.1904&0.1654&- 7.87 \\
        DTME-MTL  & \textbf{0.0150}&\textbf{0.0154}&\textbf{0.1193}&0.1733&0.1668&\textbf{0.0891}&\textbf{0.2038}&\textbf{0.7373}&\textbf{0.1567}&\textbf{0.1773}&\textbf{0.1517}&+ \textbf{4.67} \\   
        \bottomrule 
    \end{tabular}}
    \label{tab:opt_taskonomy_vitB}
    \vspace{-5pt}
\end{table*}

\begin{table}[t]
    \vspace{-5pt}
    \caption{Adaptation of DTME-MTL to other state-of-the-art MTL methods on NYUD-v2.}
    \vspace{-5pt}
    \centering
    \renewcommand\arraystretch{0.90}
    \resizebox{0.88\linewidth}{!}{
    \begin{tabular}{l|c|c|c|c}
    \hline
    Task         & Semseg & Depth  & Normal & Edge \\
    Metric       & mIoU $\uparrow$  & RMSE $\downarrow$  & mErr $\downarrow$  & odsF $\uparrow$    \\ \hline
    MTformer \citep{mtformer}   & 50.04  & 0.490 & - & - \\
    InvPT \citep{invpt}       & 53.56  & 0.5183 & 18.81  & 78.10    \\
    \rowcolor[HTML]{E0E0E0}
    + DTME-MTL       & 54.38  & \textbf{0.5020} & 18.51  & \underline{78.20}    \\
     Taskprompter \citep{ye2022taskprompter}       & \underline{55.30}  & 0.5152 & \underline{18.47}  & \underline{78.20}    \\
    \rowcolor[HTML]{E0E0E0}
    + DTME-MTL        & \textbf{56.36}  & \underline{0.5122} & \textbf{18.38}  & \textbf{78.40}    \\
    \hline
    \end{tabular}}
    \label{tab:tab_sota_1}
    \vspace{-5pt}
\end{table}
\begin{table}[t]
    \caption{Adaptation of DTME-MTL to other state-of-the-art MTL methods on PASCAL-Context.}
    \vspace{-9pt}
    \centering
    \renewcommand\arraystretch{0.96}
    \resizebox{0.96\linewidth}{!}{
        \begin{tabular}{l|c|c|c|c|c}
        \hline
        Task       & Semseg & Parsing & Saliency & Normal & Edge  \\
        Metric     & mIoU $\uparrow$  & mIoU $\uparrow$   & maxF $\uparrow$    & mErr $\downarrow$  & odsF $\uparrow$ \\ \hline
        MTformer \citep{mtformer} & 73.51 & 64.26 & 67.24 & - & - \\
        InvPT \citep{invpt}      & 79.03  & 67.61   & 84.81    & 14.15  & 73.00 \\
        \rowcolor[HTML]{E0E0E0}
        + DTME-MTL      & \textbf{81.91}  & \textbf{71.13}   & \textbf{84.96}    & 13.73  & \textbf{73.80} \\
        Taskprompter \citep{ye2022taskprompter}      & 80.89  & 68.89   & \underline{84.83}    & \underline{13.72}  & 73.50 \\
        \rowcolor[HTML]{E0E0E0}
        + DTME-MTL      & \underline{81.01}  & \underline{69.08}   & 84.75    & \textbf{13.65}  & \underline{73.60} \\
        \hline
        \end{tabular}
    }
    \label{tab:tab_sota_2}
    \vspace{-12pt}
\end{table}

\subsection{Experimental Results}
\textbf{Effectiveness of Token Modulation and Expansion.} We assess the effectiveness of the proposed methods on the NYUD-v2 and PASCAL-Context datasets, with results detailed in \cref{tab:tab_ablation}. In the last three rows of the table, we depict the performance gains compared to the two baselines and the increased number of parameters in ``$\#Param\uparrow$ ($\%$)". Compared to the Baseline (MT), our methods demonstrate significant performance improvements across all tasks in both datasets. Particularly noteworthy is the substantial increase in multi-task performance achieved with just a 0.2\% to 0.3\% increase in the total network parameters. Additionally, our approach exhibits comparable performance to Baseline (ST) in a multi-task scenario. This implies that reducing negative transfer is achievable by simply integrating token modulators and task tokens, without complex modules.

\vspace{2pt}
\noindent\textbf{Analysis of the Timing of Network Expansion.}
In \cref{fig:iter_perform}, we analyze the performance of each task according to the timing of network expansion using the proposed DTME-MTL. Specifically, the timing for expansion refers to the point at which token modulation and expansion are performed based on calculations of the token space using Singular Value Decomposition and measurement of gradient conflicts. The figure illustrates the performance results when network expansion is conducted at the beginning of training ($0^{th}$) and after each quarter of the entire training process ($i^{th}$ $25\%$ Iter). To ensure fair comparisons, we trained the network using the same number of iterations after the expansion. The results indicate that the optimal timing for expansion may vary across tasks. However, overall, early-stage expansion during network training tends to yield better performance. This aligns with the fact that our approach builds on pre-trained backbone networks.

\vspace{2pt}
\noindent\textbf{Analysis of Gradient Conflicts in Parameters.}  
We focus on resolving gradient conflicts in token space. While our primary method operates in token-level representations, we also monitor gradient conflicts in parameter space to better understand training dynamics. In \cref{fig:conflicts}, we visualize the distribution of angles between task-specific gradients of network parameters, categorizing them into different ranges and tracking their frequency over the course of training. When applied to the baseline model, both Token Modulation (TM) and Token Expansion (TE) reduce gradient conflicts in parameter space to some extent. However, as shown in \cref{tab:recon}, where our methods are compared with Recon \cite{guangyuan2022recon}, we observe important differences. Recon explicitly suppresses conflicts by modifying network parameters whenever the cosine similarity between task gradients becomes negative. Although this reduces gradient conflicts in parameter space, it often results in severe overfitting and degraded multi-task performance. These findings suggest that conflicts in parameter space are not always reliable indicators of negative transfer. Instead, resolving conflicts in token space offers an alternative strategy that avoids overfitting while still mitigating interference. Additional analysis of token-level conflicts is provided in \cref{sup:additional_analysis}.

\vspace{2pt}
\noindent\textbf{Computational Cost of DTME-MTL.} In \cref{tab:tab_time}, we report the time consumption for each process of DTME-MTL on PASCAL-Context using a single NVIDIA RTX A6000. We measure the time required for calculating the token space with SVD and for computing gradient conflicts in the token space of the transformer. The time required for each process increases with the size of the transformer. However, the proposed methods are computationally efficient, requiring approximately 1 hour with ViT-L. Considering that typical multi-task architectures require at least one day of training on a single GPU, the computational cost of DTME-MTL is relatively low. Proposed DTME-MTL increases inference time of each task about $13.4\%$ with ViT-B.

\vspace{2pt}
\noindent\textbf{Comparing Performance based on Layer Selection Criteria.}
In \cref{tab:ablation_random}, we applied TM+TE to layers with the highest gradient conflicts between tasks. Results are also shown for randomly chosen layers (Random) or layers with the lowest gradient conflicts (Reverse). The network expansion system, using conflict detection, outperforms random selection across all tasks. Particularly, applying TM+TE to layers with severe conflict levels consistently outperforms its application in layers with lower conflict levels, validating the effectiveness of the strategy.

\vspace{2pt}
\noindent\textbf{Comparison with Multi-Task Optimization.}
In \cref{tab:opt_taskonomy_vitB}, we compare DTME-MTL with previous multi-task optimization approaches to demonstrate its effectiveness in reducing negative transfer between tasks on the Taskonomy benchmark using ViT-B. DTME-MTL achieves the best multi-task performance, improving each task by an average of $4.67\%$ with only a $0.118\%$ increase in the number of parameters. Although DTME-MTL introduces additional parameters to address negative transfer, making direct comparisons with optimization methods less straightforward, it consistently improves multi-task performance. However, using more task-specific parameters does not always lead to better performance. As as shown in \cref{tab:recon}, Recon shows poor results with the vision transformer on NYUD-v2. This comparison highlights that previous multi-task expansion approaches, which naively duplicate network branches, are not only parameter-inefficient but also prone to overfitting due to the increased complexity of transformers.

\begin{table}[t]
\vspace{-5pt}
\caption{Time consumption of each process in DTME-MTL across different backbone sizes, recorded in minutes.}
\vspace{-7pt}
\scriptsize
\centering
\resizebox{0.96\linewidth}{!}{
\begin{tabular}{l|cccc}
\hline
Process (min)              & ViT-T   & ViT-S    & ViT-B    & ViT-L    \\ \hline
Calculate Token Space (SVD) & 3.61 & 3.74 & 11.54& 11.96 \\
Calculate Gradient Conflict & 8.25 & 16.43 & 21.94& 58.66 \\ \hline
\end{tabular}
}
\vspace{-10pt}
\label{tab:tab_time}
\end{table}

\begin{table}[t]
\caption{Comparison with Recon on NYUD-v2}
\vspace{-8pt}
\centering
\renewcommand\arraystretch{1.00}
\scriptsize
\begin{tabular}{l|ccccc}
\hline
\multicolumn{1}{c|}{} & Semseg  & Depth  & Normal  & Edge & \#Param $\uparrow$ (\%) \\ 
\multicolumn{1}{c|}{\multirow{-2}{*}{Method}} & mIoU $\uparrow$  & RMSE $\downarrow$  & mErr $\downarrow$  & odsF $\uparrow$ & \\ \hline
Joint      &34.13  &0.673  &22.51  &56.38  &    0.0\\
Recon \citep{guangyuan2022recon}     &31.92  &0.693  &23.35  &52.80  &    23.34\\
Ours       &\textbf{38.27}  &\textbf{0.6370}  &\textbf{21.64}  & \textbf{57.90}  &  \textbf{0.24}\\ \hline
\end{tabular}
\vspace{-10pt}
\label{tab:recon}
\end{table}

\vspace{2pt}\noindent
\textbf{Adapting to Multi-Task Architectures.} In \cref{tab:tab_sota_1,tab:tab_sota_2}, we compare DTME-MTL with leading multi-task architectures on the NYUD-v2 and PASCAL-Context datasets. We evaluate its multi-task performance against transformer-based approaches. Our method is compatible with any transformer-based multi-task architecture, enabling us to assess its effectiveness by integrating it into two leading models: InvPT and TaskPrompter. DTME-MTL seamlessly enhances these architectures, significantly boosting performance with only a minimal increase in parameters — just $0.048\%$ for InvPT and $0.046\%$ for TaskPrompter.
\section{Conclusion}
We introduced Dynamic Token Modulation and Expansion for Multi-Task Learning, an efficient approach for improving transformer-based MTL architectures. By categorizing gradient conflicts into range space and null space, it adaptively applies token modulation and expansion to mitigate negative transfer and reduce overfitting. DTME-MTL seamlessly integrates with existing transformer-based MTL frameworks, requiring minimal additional parameters. By refining encoded token space, it provides a lightweight and scalable solution for enhancing multi-task performance.

\section*{Acknowledgment}
This research was supported by the Challengeable Future Defense Technology Research and Development Program through the Agency For Defense Development(ADD) funded by the Defense Acquisition Program Administration(DAPA) in 2025(No.915102201).

{
    \small
    \bibliographystyle{ieeenat_fullname}
    \bibliography{main}
}

\appendix
\maketitlesupplementary
\section{Additional Related Works}
\textbf{Multi-Task Architectures.} Various multi-task architectures can be categorized based on how the parameters or features of the sharing network are distributed among tasks. The widely used shared trunk structure comprises a common encoder shared by multiple tasks and a dedicated decoder for each task \citep{RN51, RN52, RN49, RN50}. A tree-like architecture, with multiple division points for each task group, offers a more generalized structure \citep{treelike1, treelike2, treelike3, treelike4}. The cross-talk architecture employs separate symmetrical networks for each task, utilizing feature exchange between layers at the same depth for information sharing between tasks \citep{RN43, RN29}. The prediction distillation model \citep{RN9, RN29, RN32, pap} incorporates cross-task interactions at the end of the shared encoder, while the task switching network \citep{RN30, RN40, RN42, RN2} changes network parameters depending on the task.

\section{Experimental Settings}
\label{Append:exp_settings}
\subsection{Datasets}
We evaluate our method on three benchmarks: NYUD-v2, PASCAL-Context, and Taskonomy. NYUD-v2 contains 4 vision tasks: Our evaluation is based on depth estimation, semantic segmentation, surface normal prediction, and edge detection. PASCAL-Context contains 5 tasks: We evaluate semantic segmentation, human parts estimation, saliency estimation, surface normal prediction, and edge detection. We used 11 tasks for Taskonomy: We evaluate Depth Euclidean (DE), Depth Zbuffer (DZ), Edge Texture (ET),  Keypoints 2D (Key2D), Keypoints 3D (Key3D), Normal (N), Principal Curvature (PC), Reshading (R), Segment Unsup 2d (S2D), and Segment Unsup 2.5D (S25D).

\subsection{Implementation Details}
For experiments, we adopt ViT \citep{vit} pre-trained on ImageNet-22K \citep{deng2009imagenet} as the transformer encoder. The models are trained for 60,000 iterations on both NYUD \citep{RN15} and PASCAL \citep{RN12} datasets with batch size 6. We use Adam optimizer with learning rate $2\times$$10^{-5}$ and $1\times$$10^{-6}$ of a weight decay with a polynomial learning rate schedule. Following the previous works \citep{invpt, ye2022taskprompter}, the cross-entropy loss is used for semantic segmentation, human parts estimation, and saliency, edge detection. Surface normal prediction and depth estimation use L1 loss.

\subsection{Design and Implementation Strategy}
To improve efficiency, we perform SVD only once early in training to estimate the feature space for conflict analysis. Gradient conflicts are measured in a pairwise manner across tasks, and the average number of conflicts in each space is used to guide token expansion. Based on this, we statically allocate a small number of task-specific tokens (six in our setup) as learnable parameters, independently applied at each layer. These tokens are fixed during training and do not adapt dynamically. For NYUD-v2 and PASCAL-Context, we use the full training sets to compute gradient statistics, while for Taskonomy, covariance is estimated using 100 randomly sampled mini-batches. The assignment of Token Modulation (TM) and Token Expansion (TE) is determined by a manually chosen activation ratio, which we analyze in \cref{fig:sup:prop}. Rather than activating all components uniformly, TM and TE are selectively applied to layers with the highest conflict levels, either individually or jointly, based on their effectiveness in reducing task interference.

\subsection{Evaluation}
For semantic segmentation, we utilize mean Intersection over Union (mIoU). Surface normal prediction performance is measured by the mean angular distance between the predicted output and ground truth. Depth estimation is evaluated using Root Mean Squared Error (RMSE). For saliency estimation and human part segmentation, we employ mIoU. Edge detection is assessed using the optimal-dataset-scale F-measure (odsF). For Taskonomy, we adopt RMSE for principal curvature and L1 distance for the remaining tasks.

\begin{table*}[t!]
    \centering
    \vspace{-5pt}
    \caption{Comparison with multi-task optimization approaches on Taskonomy across 11 different tasks with ViT-L. Non-converged results are indicated with a dash.}
    \resizebox{\textwidth}{!}{
    \begin{tabular}{l|ccccccccccc|c}
    \toprule
        Task       & DE     & DZ      & EO     & ET     & Key2D    & Key3D   & N      & PC     & R      & S2D    & S25D   &    \\
        Metric    & L1 Dist. $\downarrow$  & L1 Dist. $\downarrow$ & L1 Dist. $\downarrow$ & L1 Dist. $\downarrow$ & L1 Dist. $\downarrow$ & L1 Dist. $\downarrow$ & L1 Dist. & RMSE $\downarrow$    & L1 Dist. $\downarrow$ & L1 Dist. $\downarrow$ & L1 Dist. $\downarrow$  & $\triangle_m$ $\uparrow$ $(\%)$ \\ \midrule 
        ST    & 0.0141&0.0146&0.0992&0.1716&0.1631&0.0801&0.2133&0.7134&0.1342&0.1688&0.1419& 0.00 \\ \midrule
        GD        & 0.0153&0.0156&0.1196&0.1757&0.1729&0.0896&0.2215&0.7451&0.1576&0.1826&0.1537&-8.92 \\
        GradDrop  & 0.0170&0.0195&0.1235&0.1757&0.1753&0.0909&0.2818&0.7679&0.1663&0.1916&0.1543&-17.07  \\
        MGDA      & - & - & - & - & - & - & - & - & - & - & - & - \\
        UW        & 0.0152&0.0155&0.1195&0.1755&0.1728&0.0897&0.2356&0.7436&0.1569&0.1830&0.1538&-9.36 \\
        DWA       & 0.0153&0.0156&0.1197&0.1757&0.1730&0.0897&0.2214&0.7441&0.1576&0.1827&0.1537&-8.96 \\
        PCGrad    & 0.0152&0.0156&0.1192&0.1749&0.1699&0.0893&0.2310&0.7475&0.1577&0.1825&0.1480&-8.63 \\  
        CAGrad    & 0.0155&0.0156&0.1175&0.1756&0.1649&0.0860&0.2421&0.7544&0.1591&0.1854&0.1554&-9.32 \\
        IMTL      & 0.0151&0.0156&0.1194&0.1755&0.1726&0.0895&0.2199&0.7432&0.1569&0.1824&0.1533&-8.57 \\
        Align-MTL & 0.0150&0.0155&0.1136&0.1733&0.1633&0.0862&0.2512&0.8029&0.1643&0.1803&0.1445&-8.78 \\
        Nash-MTL  & 0.0151&0.0154&0.1138&0.1732&0.1644&0.0863&0.2507&0.7656&0.1544&0.1833&0.1452&-7.95 \\
        FAMO      & 0.0153&0.0157&0.1196&0.1757&0.1730&0.0897&0.2221&0.7444&0.1575&0.1830&0.1534&-8.99 \\
        DTME-MTL  & 0.0127&0.0130&0.1088&0.1731&0.1665&0.0852&0.1654&0.6890&0.1389&0.1661&0.1404&+2.41 \\   
        \bottomrule 
    \end{tabular}}
	\label{tab:opt_taskonomy_vitL}
    \vspace{-5pt}
\end{table*}
\begin{table*}[t!]
    \centering
    \caption{Comparison with multi-task optimization approaches on Taskonomy across 11 different tasks with ViT-S. Non-converged results are indicated with a dash.}
    \vspace{-5pt}
    \resizebox{\textwidth}{!}{
    \begin{tabular}{l|ccccccccccc|c}
    \toprule
        Task       & DE     & DZ      & EO     & ET     & Key2D    & Key3D   & N      & PC     & R      & S2D    & S25D   &    \\
        Metric    & L1 Dist. $\downarrow$  & L1 Dist. $\downarrow$ & L1 Dist. $\downarrow$ & L1 Dist. $\downarrow$ & L1 Dist. $\downarrow$ & L1 Dist. $\downarrow$ & L1 Dist. & RMSE $\downarrow$    & L1 Dist. $\downarrow$ & L1 Dist. $\downarrow$ & L1 Dist. $\downarrow$  & $\triangle_m$ $\uparrow$ $(\%)$ \\ \midrule 
        ST        0.0255&0.0255&0.1285&0.1727&0.1653&0.0918&0.3973&0.8562&0.1864&0.1824&0.1647& 0.00 \\ \midrule
        GD        &0.0244&0.0243&0.1501&0.1778&0.1844&0.1009&0.4105&0.9087&0.2325&0.2032&0.1822&-8.04 \\
        GradDrop  &0.0253&0.0253&0.1533&0.1785&0.1865&0.1021&0.4399&0.9246&0.2408&0.2063&0.1791&-10.42 \\
        MGDA      & - & - & - & - & - & - & - & - & - & - & - & - \\
        UW        &0.0242&0.0242&0.1498&0.1778&0.1847&0.1007&0.4064&0.9079&0.2312&0.2033&0.1822&-7.74 \\
        DWA       &0.0242&0.0242&0.1500&0.1778&0.1844&0.1008&0.4097&0.9071&0.2316&0.2032&0.1822&-7.84 \\
        PCGrad    &0.0248&0.0248&0.1501&0.1755&0.1761&0.1001&0.4306&0.9181&0.2371&0.2023&0.1772&-8.12 \\  
        CAGrad    &0.0254&0.0255&0.1516&0.1738&0.1698&0.0983&0.4535&0.9282&0.2442&0.2068&0.1849&-9.74 \\
        IMTL      &0.0236&0.0237&0.1456&0.1756&0.1760&0.0988&0.4151&0.9055&0.2222&0.2010&0.1794&-5.74 \\
        Align-MTL &0.0266&0.0264&0.1499&0.1736&0.1700&0.0986&0.4659&0.9868&0.2604&0.2030&0.1780&-11.51 \\
        Nash-MTL  &0.0235&0.0235&0.1432&0.1745&0.1718&0.0975&0.4230&0.9225&0.2268&0.1985&0.1775&-5.41 \\
        FAMO      &0.0243&0.0243&0.1499&0.1778&0.1846&0.1008&0.3841&0.9080&0.2321&0.2027&0.1816&-7.31 \\
        DTME-MTL  &0.0196&0.0200&0.1372&0.1754&0.1712&0.0958&0.3129&0.8333&0.1955&0.1907&0.1698&+3.62 \\   
        \bottomrule 
    \end{tabular}}
	\label{tab:opt_taskonomy_vitS}
 \vspace{-5pt}
\end{table*}
\begin{table*}[t!]
    \centering
    \caption{Comparison with multi-task optimization approaches on Taskonomy across 11 different tasks with ViT-T. Non-converged results are indicated with a dash.}
    \vspace{-5pt}
    \resizebox{\textwidth}{!}{
    \begin{tabular}{l|ccccccccccc|c}
    \toprule
        Task       & DE     & DZ      & EO     & ET     & Key2D    & Key3D   & N      & PC     & R      & S2D    & S25D   &    \\
        Metric    & L1 Dist. $\downarrow$  & L1 Dist. $\downarrow$ & L1 Dist. $\downarrow$ & L1 Dist. $\downarrow$ & L1 Dist. $\downarrow$ & L1 Dist. $\downarrow$ & L1 Dist. & RMSE $\downarrow$    & L1 Dist. $\downarrow$ & L1 Dist. $\downarrow$ & L1 Dist. $\downarrow$  & $\triangle_m$ $\uparrow$ $(\%)$ \\ \midrule 
        ST        &0.0250&0.0256&0.1388&0.1755&0.1670&0.0958&0.3856&0.9066&0.2132&0.1878&0.1722& 0.00 \\ \midrule
        GD        &0.0266&0.0278&0.1593&0.1794&0.1865&0.1047&0.4752&0.9467&0.2568&0.2081&0.1897&-11.10 \\
        GradDrop  &0.0276&0.0284&0.1624&0.1807&0.1884&0.1064&0.4741&0.9611&0.2658&0.2108&0.1860&-12.67 \\
        MGDA      & - & - & - & - & - & - & - & - & - & - & - & - \\
        UW        &0.0266&0.0277&0.1593&0.1795&0.1865&0.1045&0.4757&0.9466&0.2567&0.2080&0.1896&-11.07 \\
        DWA       &0.0266&0.0274&0.1593&0.1794&0.1866&0.1045&0.4743&0.9465&0.2567&0.2080&0.1897&-10.95 \\
        PCGrad    &0.0273&0.0285&0.1596&0.1768&0.1807&0.1043&0.4785&0.9689&0.2644&0.2080&0.1854&-11.55 \\  
        CAGrad    &0.0290&0.0305&0.1641&0.1747&0.1731&0.1051&0.4884&0.9870&0.2828&0.2136&0.1945&-14.64 \\
        IMTL      &0.0263&0.0272&0.1558&0.1772&0.1810&0.1025&0.4730&0.9525&0.2458&0.2065&0.1868&-9.24 \\
        Align-MTL & - & - & - & - & - & - & - & - & - & - & - & - \\
        Nash-MTL  &0.0261&0.0270&0.1536&0.1762&0.1766&0.1017&0.4590&0.9649&0.2496&0.2039&0.1846&-8.28 \\
        FAMO      &0.0266&0.0275&0.1592&0.1795&0.1865&0.1047&0.4746&0.9466&0.2566&0.2080&0.1898&-10.97 \\
        DTME-MTL  &0.0236&0.0241&0.1494&0.1765&0.1790&0.0998&0.4138&0.8921&0.2290&0.1959&0.1824&-2.88 \\   
        \bottomrule 
    \end{tabular}}
	\label{tab:opt_taskonomy_vitT}
\end{table*}
\begin{table*}[t!]
\caption{We compare task performance based on the configuration of the modulator. Specifically, we compare the performance of tasks using an affine transformation against those using a batch normalization layer as configurations for the modulator.}
\vspace{-5pt}
\centering
\renewcommand\arraystretch{1.10}
\tiny
\resizebox{\textwidth}{!}{
\begin{tabular}{c|cccc|ccccc}
\hline
\multirow{3}{*}{}         & \multicolumn{4}{c|}{NYUD-v2}      & \multicolumn{5}{c}{PASCAL-Context}   \\ \cline{2-10}
                                & \multicolumn{1}{c|}{Semseg} & \multicolumn{1}{c|}{Depth}  & \multicolumn{1}{c|}{Normal} & Edge  & \multicolumn{1}{c|}{Semseg} & \multicolumn{1}{c|}{Parsing} & \multicolumn{1}{c|}{Saliency} & \multicolumn{1}{c|}{Normal} & Edge  \\
\multicolumn{1}{c|}{\multirow{-3}{*}{Model}}                             & \multicolumn{1}{c|}{mIoU $\uparrow$}   & \multicolumn{1}{c|}{RMSE $\downarrow$}   & \multicolumn{1}{c|}{mErr $\downarrow$}   & odsF $\uparrow$ & \multicolumn{1}{c|}{mIoU $\uparrow$}   & \multicolumn{1}{c|}{mIoU $\uparrow$}    & \multicolumn{1}{c|}{maxF $\uparrow$}     & \multicolumn{1}{c|}{mErr $\downarrow$}   & odsF $\uparrow$ \\ \hline
\multicolumn{1}{l|}{TM+TE (Affine)}     & \multicolumn{1}{c|}{\textbf{38.27}}  & \multicolumn{1}{c|}{\textbf{0.6370}} & \multicolumn{1}{c|}{\textbf{21.64}}  & \textbf{57.90} & \multicolumn{1}{c|}{\textbf{66.18}}  & \multicolumn{1}{c|}{\textbf{56.29}}   & \multicolumn{1}{c|}{\textbf{83.21}}    & \multicolumn{1}{c|}{\textbf{15.26}}  & \textbf{47.00} \\ \hline
\multicolumn{1}{l|}{TM+TE (Batch Norm)} & \multicolumn{1}{c|}{37.42}  & \multicolumn{1}{c|}{0.6550} & \multicolumn{1}{c|}{23.16}  & 56.10 & \multicolumn{1}{c|}{60.80}  & \multicolumn{1}{c|}{53.29}   & \multicolumn{1}{c|}{82.59}    & \multicolumn{1}{c|}{15.73}  & 44.90 \\ \hline
\end{tabular}}
\label{tab:sup:ablation_mod_config}
\end{table*}

\section{Additional Experiments}
\label{Append:additional_experimemts}
\textbf{Comparison with Multi-Task Optimization.} In \cref{tab:opt_taskonomy_vitL,tab:opt_taskonomy_vitS,tab:opt_taskonomy_vitT}, we further evaluate the proposed DTME-MTL against previous multi-task optimization approaches using different backbone sizes. Our method demonstrates significant improvements in multi-task performance with minimal increases in parameters. Specifically, DTME-MTL results in a parameter increase of $0.089\%$ for ViT-L, $0.23\%$ for ViT-S, and $0.46\%$ for ViT-T.

\begin{table*}[t!]
\caption{We assess task performance by comparing scenarios where we freeze the backbone network after expansion (w/ Freeze) and where we don't (w/o Freeze).}
\vspace{-5pt}
\centering
\renewcommand\arraystretch{1.00}
\tiny
\resizebox{\textwidth}{!}{
\begin{tabular}{c|cccc|ccccc}
\hline
\multirow{3}{*}{}         & \multicolumn{4}{c|}{NYUD-v2}         & \multicolumn{5}{c}{PASCAL-Context}     \\ \cline{2-10}
                                & \multicolumn{1}{c|}{Semseg} & \multicolumn{1}{c|}{Depth}  & \multicolumn{1}{c|}{Normal} & Edge  & \multicolumn{1}{c|}{Semseg} & \multicolumn{1}{c|}{Parsing} & \multicolumn{1}{c|}{Saliency} & \multicolumn{1}{c|}{Normal} & Edge  \\
\multicolumn{1}{c|}{\multirow{-3}{*}{Model}}                             & \multicolumn{1}{c|}{mIoU $\uparrow$}   & \multicolumn{1}{c|}{RMSE $\downarrow$}   & \multicolumn{1}{c|}{mErr $\downarrow$}   & odsF $\uparrow$ & \multicolumn{1}{c|}{mIoU $\uparrow$}   & \multicolumn{1}{c|}{mIoU $\uparrow$}    & \multicolumn{1}{c|}{maxF $\uparrow$}     & \multicolumn{1}{c|}{mErr $\downarrow$}   & odsF $\uparrow$ \\ \hline
\multicolumn{1}{l|}{TM+TE (w/ Freeze)} & \multicolumn{1}{c|}{34.80}  & \multicolumn{1}{c|}{0.6730} & \multicolumn{1}{c|}{22.48}  & 56.00 & \multicolumn{1}{c|}{58.34}  & \multicolumn{1}{c|}{52.96}   & \multicolumn{1}{c|}{82.86}    & \multicolumn{1}{c|}{15.63}  & 43.20 \\ \hline
\multicolumn{1}{l|}{TM+TE (w/o Freeze)}     & \multicolumn{1}{c|}{\textbf{38.27}}  & \multicolumn{1}{c|}{\textbf{0.6370}} & \multicolumn{1}{c|}{\textbf{21.64}}  & \textbf{57.90} & \multicolumn{1}{c|}{\textbf{66.18}}  & \multicolumn{1}{c|}{\textbf{56.29}}   & \multicolumn{1}{c|}{\textbf{83.21}}    & \multicolumn{1}{c|}{\textbf{15.26}}  & \textbf{47.00} \\ \hline

\end{tabular}}
\label{tab:sup:ablation_backbone_freeze}
\vspace{-10pt}
\end{table*}

\begin{figure*}[t!]
\centering
    \begin{subfigure}{\linewidth}
    \includegraphics[width=0.99\columnwidth]{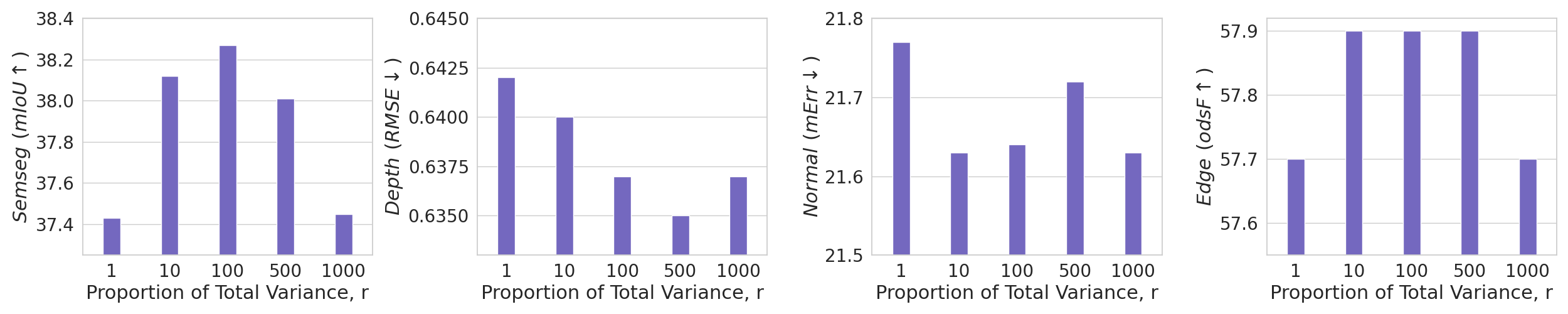}
    \caption{Results on NYUD-v2.}
    \label{fig:sup:svd_nyud}
    \end{subfigure}
    
    \begin{subfigure}{\linewidth}
    \includegraphics[width=0.99\columnwidth]{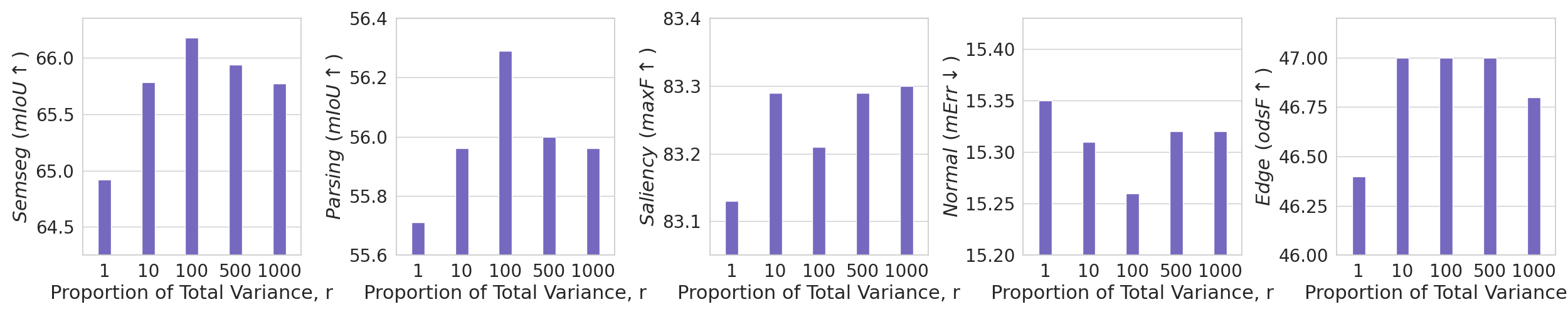}
    \caption{Results on PASCAL-Context.}
    \label{fig:sup:svd_pascal}
    \end{subfigure}
\vspace{-18pt}
\caption{We assess the performance of tasks based on the proportion of total variance $r$. The results are displayed for both (a) NYUD-v2 and (b) PASCAL-Context.}
\label{fig:sup:influence_r}
\vspace{-8pt}
\end{figure*}

\begin{figure*}[t!]
\vspace{-8pt}
\centering
    \begin{subfigure}{\linewidth}
    \includegraphics[width=0.99\columnwidth]{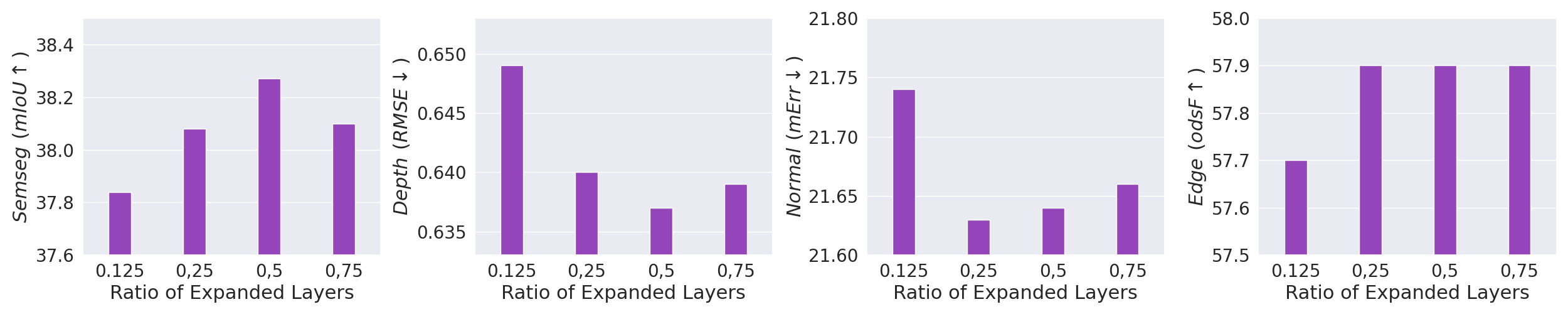}
    \caption{Results on NYUD-v2.}
    \label{fig:sup:prop_nyud}
    \end{subfigure}
    \begin{subfigure}{\linewidth}
    \includegraphics[width=0.99\columnwidth]{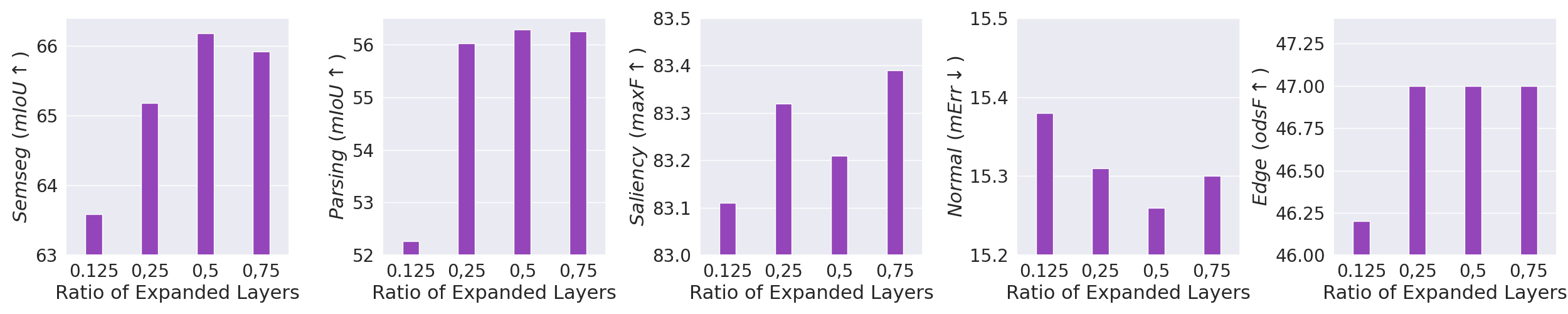}
    \caption{Results on PASCAL-Context.}
    \label{fig:sup:prop_pascal}
    \vspace{-8pt}
    \end{subfigure}
\caption{The performance of tasks based on the ratio of the number of expanded layers to the total number of layers. The results are displayed for both (a) NYUD-v2 and (b) PASCAL-Context.}
\label{fig:sup:prop}
\vspace{-10pt}
\end{figure*}

\noindent\textbf{Analysis on the Modulator Configuration.} In \cref{tab:sup:ablation_mod_config}, we show the performance difference based on the configuration of the token modulators. Specifically, we compared the outcomes obtained when employing affine transformation and batch normalization, which could be considered as the most common and straightforward approaches. Through experiments, we find that affine transformations consistently exhibit better performance across all tasks compared to batch normalization layers used as modulators for both datasets.

\noindent\textbf{Analyzing Performance Differences with Backbone Network Freezing.} In \cref{tab:sup:ablation_backbone_freeze}, we examine the performance variation based on whether we freeze the existing backbone network components when training the expanded network after implementing the proposed dynamic token modulation and expansion. The results indicate that training networks without freezing the existing backbone network components leads to significantly better performance compared to training networks with freezing. We guess that allowing modifications to the learned token space after expansion helps the network to dynamically partition the token space for each task.

\noindent\textbf{Influence of $\textbf{r}$ on SVD Approximation.} In \cref{fig:sup:influence_r}, we illustrate how the proportion of total variance $r$ impacts the approximation of a token's range and null space. We assess the performance of tasks across five values of $r$ (1, 10, 100, 500, 1000). Our results suggest that the value of $r$ has minimal impact on task performance, implying that there is less need for extensive tuning of the $r$ parameter to optimize performance. In our other experiments, we chose $r$ as 100 for training.

\noindent\textbf{Impact of the Number of Layers Expanded by DTME-MTL.}  
DTME-MTL expands a subset of transformer layers selected based on the severity of gradient conflicts. In \cref{fig:sup:prop}, we analyze how varying the number of expanded layers affects task performance. The x-axis denotes the ratio of expanded layers to the total number of layers. We observe that applying TM+TE to approximately 25\%–50\% of the layers yields consistent performance gains across tasks while maintaining parameter efficiency. Performance improves as more high-conflict layers are expanded, but begins to degrade when expansion exceeds 50\%, especially when low-conflict layers are included. This suggests that over-expansion can be detrimental. \Cref{tab:ablation_random} further confirms that using a moderate expansion ratio (50\%) avoids overfitting, whereas \cref{fig:sup:prop} highlights that indiscriminate expansion into less conflicting layers harms performance. These findings underscore the importance of both the \textit{extent} and \textit{location} of TM+TE application.

\noindent\textbf{Effect of Swapping Conflict Types.} In \cref{tab:sel}, we present the results of an experiment on NYUD-v2 where we intentionally swap the conflict types targeted by each method. Specifically, Token Expansion (TE) is applied to layers with severe range space conflict, and Token Modulation (TM) is applied to layers with severe null space conflict—opposite to our standard configuration. This reversal leads to a clear performance drop, confirming that each method is most effective when applied to the type of conflict it is designed to resolve. These results support our design choice of assigning TM to range space conflict and TE to null space conflict.

\begin{table}[h]
\centering
\small
\vspace{-4pt}
\caption{Performance comparison across selection strategies.}
\vspace{-8pt}
\begin{tabular}{c|cccc} \hline
Method                              & Random & Reverse & Swap   & Standard \\\hline
$\triangle_m$ $\uparrow$ $(\%)$     & -2.966 & -6.167  & -2.608 & +0.044  \\\hline
\end{tabular}
\label{tab:sel}
\vspace{-4pt}
\end{table}

\section{Additional Analysis}
\label{sup:additional_analysis}
\noindent\textbf{Further Justification for Targeted TM/TE Assignment.}  
Prior work \cite{neyshabur2020being} suggests that fine-tuning from a pretrained model tends to remain in the same loss basin, preserving the structure of the pretrained feature space. Accordingly, we view the token space during fine-tuning as constrained by the span of the pretrained features. If the conflict lies within this span (i.e., the range space), it can be resolved by rotating the token space—achievable via a modulator, since affine transformations include rotation. However, if the conflict resides in the null space, it lies outside the span and cannot be sufficiently addressed by modulation alone. In such cases, expanding the token space with task-specific tokens helps relax this constraint.   We theoretically support this view in Propositions~\ref{theorem1} and~\ref{theorem2} (with proofs in \cref{Append:proof}), which analyze how each method addresses conflict in its respective subspace. This is further validated empirically: we measure the reduction in gradient conflicts by comparing the start and end of training in each space (\cref{tab:conflict}). The results show that Token Modulation (TM) is more effective in reducing conflicts in the range space, while Token Expansion (TE) is more effective in the null space. This consistency between theoretical analysis and empirical behavior supports our design choice to selectively apply TM and TE based on the dominant type of conflict in each layer.

\noindent\textbf{Token-Level vs. Parameter-Level Conflict Handling.}  
Parameter-space conflicts reflect an aggregate gradient across all tokens, which makes it difficult to localize or disentangle the source of interference. In contrast, token-level conflicts can be measured for each individual token, allowing more localized and fine-grained analysis. This granularity enables our method to selectively modulate or expand tokens based on where the conflict occurs. Furthermore, by decomposing the token space into range and null components—depending on whether the pretrained model already spans those directions—we adaptively apply Token Modulation (TM) or Token Expansion (TE) to address conflicts. Such space-aware conflict resolution is fundamentally infeasible in parameter space, where task interference is entangled across layers and tokens.

\noindent\textbf{Comparison with LoRA in Multi-Task Inference.}  
As shown in \cref{tab:tab_ablation}, the baseline (ST) corresponds to full fine-tuning and serves as an upper bound on performance. While LoRA \cite{hu2022lora} is a parameter-efficient method, assigning a separate LoRA module for each task leads to disjoint sets of task-specific weights. Even when merged into the base model, these configurations require separate forward passes per task, negating the efficiency benefits of multi-task learning (MTL). In contrast, our method maintains shared weights across tasks and allows all outputs to be computed jointly in a single batched tensor operation on GPU. This enables highly parallel inference with only a 13.4\% overhead per task, whereas the inference time in LoRA scales linearly with the number of tasks.

\begin{table}[h]
\vspace{-5pt}
\centering
\small
\caption{Reduction in gradient conflict numbers (NYUD-v2).}
\vspace{-8pt}
\begin{tabular}{c|cccc} \hline
Method  & Num($g_{\mathcal{R},i} \cdot g_{\mathcal{R},i} \leq 0$)  & Num($g_{\mathcal{N},i} \cdot g_{\mathcal{N},i} \leq 0$) \\ \hline
TM      &11.60 $\% \downarrow$  &8.92 $\% \downarrow$  \\\hline
TE      &4.64$\% \downarrow$  &15.44 $\% \downarrow$  \\\hline
\end{tabular}
\label{tab:conflict}
\vspace{-2pt}
\end{table}

\onecolumn
\section{Theoretical Analysis}
\label{Append:proof}
\subsection{Proof of \cref{theorem1}}

\theomone*
\begin{proof}
Let the loss function $\mathcal{L}_i$ be a function of the shared parameters $\Theta_s$, the token modulator $\mathcal{M}_i$, and the input data $\mathcal{X}_l$. Since transformers convert input data into tokens, we consider the loss to be a function of one of the input tokens, $\mathcal{T}_{in}$, rather than $\mathcal{X}^l$. To represent the updating step during optimization, we use the superscript $t$ for current variables, such as $\Theta_s^t$, and $\mathcal{M}_i^t$, and $t+1$ for the next-step variables, such as $\Theta_s^{t+1}$, and $\mathcal{M}_i^{t+1}$.

In cases where the input token $\mathcal{T}_{in}$ spans the row space of $\Tilde{\mathcal{T}_s}$, this can be expressed as follows:
\begin{align}
\mathcal{U}_\mathcal{N} \mathcal{U}_\mathcal{N}^T \nabla_{\mathcal{T}_{in}} \mathcal{L}_i (\Theta_s^{t}, \mathcal{M}_i^{t}, \mathcal{T}_{in}) \simeq 0
\label{eq:null_space}
\end{align}

Since the row space and null space are perpendicular to each other, with their dimensions summing to the entire space, the following holds according to \cref{eq:null_space}:
\begin{align}
\sum_{i=1}^{\mathcal{K}} \nabla_{\mathcal{T}_{in}} \mathcal{L}_i = \sum_{i=1}^{\mathcal{K}}  (\mathcal{U}_\mathcal{R} \mathcal{U}_\mathcal{R}^T + \mathcal{U}_\mathcal{N} \mathcal{U}_\mathcal{N}^T) \nabla_{\mathcal{T}_{in}} \mathcal{L}_i
\simeq \sum_{i=1}^{\mathcal{K}}  (\mathcal{U}_\mathcal{R} \mathcal{U}_\mathcal{R}^T) \nabla_{\mathcal{T}_{in}} \mathcal{L}_i
\end{align}

Let the token modulator $\mathcal{M}_i$ be a $p \times p$ matrix that manipulates the input token $\mathcal{T}_{in}$.
\begin{align}
\sum_{i=1}^{\mathcal{K}} \nabla_{\mathcal{T}_{in}} \mathcal{L}_i
    &=\sum_{i=1}^{\mathcal{K}} 
    (\mathcal{U}_\mathcal{R} \mathcal{M}_i^t) (\mathcal{U}_\mathcal{R} \mathcal{M}_i^t)^T \cdot
    \nabla_{\mathcal{M}_i^t} \mathcal{L}_i \cdot \nabla_{\mathcal{T}_{in}} \mathcal{M}_i^t
    \label{eq:exp_opt}
\end{align}

The total multi-task loss can be represented using a Taylor expansion. Assuming $\eta \ll 1$, we can ignore the second-order terms of $\eta$:
\begin{align}
    \sum_{i=1}^{\mathcal{K}} \mathcal{L}_i (\Theta_s^{t+1}, \mathcal{M}_i^{t+1}, \mathcal{T}_s) = 
    \sum_{i=1}^{\mathcal{K}} \mathcal{L}_i (\Theta_s^{t}, \mathcal{M}_i^{t}, \mathcal{T}_s) + \sum_{i=1}^{\mathcal{K}} \nabla_{\Theta_s^t} \mathcal{L}_i (\Theta_s^t, \mathcal{M}_i^t, \mathcal{T}_s)(\Theta_s^{t+1}-\Theta_s^t) \\
    + \sum_{i=1}^{\mathcal{K}} \nabla_{\mathcal{M}_i^t} \mathcal{L}_i (\Theta_s^t, \mathcal{M}_i^t, \mathcal{T}_s)(\mathcal{M}_i^{t+1}-\mathcal{M}_i^t) \\
    = \sum_{i=1}^{\mathcal{K}} \mathcal{L}_i (\Theta_s^{t}, \mathcal{M}_i^{t}, \mathcal{T}_s) - \eta | \sum_{i=1}^{\mathcal{K}} \nabla_{\Theta_s} \mathcal{L}_i (\Theta_s^t, \mathcal{M}_i^t, \mathcal{T}_s)|^2 \\
    - \eta \sum_{i=1}^{\mathcal{K}} |\nabla_{\mathcal{M}_i^t} \mathcal{L}_i (\Theta_s^t, \mathcal{M}_i^t, \mathcal{T}_s)|^2
    \label{eq:mod_mtloss}
\end{align}

By optimizing the modulator $\mathcal{M}_i^t$ so that $|\nabla_{\mathcal{M}_i^t} \mathcal{L}_i (\Theta_s^t, \mathcal{M}_i^t, \mathcal{T}_{in})|$ approaches zero for each task $i = 1, 2, \dots, \mathcal{K}$, we can alleviate gradient conflicts in the row space of $\Tilde{\mathcal{T}}_s$ (as \cref{eq:exp_opt} also approaches zero) and reduce the overall multi-task loss, since \cref{eq:mod_mtloss} is always greater than or equal to zero.
\end{proof}

\subsection{Proof of \cref{theorem2}}
\theomtwo*
\begin{proof}
Let the loss function $\mathcal{L}_i$ be a function of the shared parameters $\Theta_s^t$, the task-specific token $\mathcal{T}_i^t$, and the input data $\mathcal{X}^t$. Similarly, since transformers convert input data into tokens, we consider the loss to be a function of one of the input tokens, $\mathcal{T}_{in}^t$, rather than $\mathcal{X}^t$. To represent the updating step during optimization, we use the superscript $t$ for current variables, such as $\Theta_s^t$, $\mathcal{T}_{in}^t$ and $\mathcal{M}_i^t$, and $t+1$ for the next-step variables, such as $\Theta_s^{t+1}$, $\mathcal{T}_{in}^{t+1}$ and $\mathcal{M}_i^{t+1}$. 

In the case where the input token $\mathcal{T}_{in}^t$ spans the null space of $\Tilde{\mathcal{T}_s}$, this can be expressed as follows:
\begin{align}
\sum_{i=1}^{\mathcal{K}} \mathcal{U}_\mathcal{R} \mathcal{U}_\mathcal{R}^T \nabla_{\mathcal{T}_{in}^t} \mathcal{L}_i (\Theta_s^{t}, \mathcal{T}_{in}^t, \mathcal{T}_i^{t}) \simeq 0
\label{eq:row_space}
\end{align}

The derivative of the task-specific loss $\mathcal{L}_i$ with respect to the expanded token, including the input token $\mathcal{T}_{in}^t$ and the learnable task-specific tokens $\mathcal{T}_i^t$, is given as follows:
\begin{align}
\sum_{i=1}^{\mathcal{K}} & \nabla_{[\mathcal{T}_{in}^t,\mathcal{T}_i^t]} \mathcal{L}_i \\
    &=\sum_{i=1}^{\mathcal{K}} 
    \Bigg(
    \begin{bmatrix}
    \mathcal{U}_\mathcal{R} & 0_{d\times\mathcal{K}} \\
    0_{\mathcal{K}\times d} &  \mathcal{U}_{\mathcal{R},i}  \\
   \end{bmatrix}
   \begin{bmatrix}
    \mathcal{U}_\mathcal{R} & 0_{d\times\mathcal{K}} \\
    0_{\mathcal{K}\times d} &  \mathcal{U}_{\mathcal{R},i}  \\
   \end{bmatrix}^{T}
   +
   \begin{bmatrix}
   \mathcal{U}_\mathcal{N} & 0_{d\times\mathcal{K}} \\
    0_{\mathcal{K}\times d} & 0_{\mathcal{K}\times\mathcal{K}} \\
   \end{bmatrix}
   \begin{bmatrix}
    \mathcal{U}_\mathcal{N} & 0_{d\times\mathcal{K}} \\
    0_{\mathcal{K}\times d} & 0_{\mathcal{K}\times\mathcal{K}} \\
   \end{bmatrix}^{T}
   \Bigg)
   \begin{bmatrix}
    \nabla_{\mathcal{T}_{in}^t} \mathcal{L}_i \\
    \nabla_{\mathcal{T}_i^t} \mathcal{L}_i  \\
   \end{bmatrix}\\
   &=\sum_{i=1}^{\mathcal{K}} 
   \begin{bmatrix}
    \mathcal{U}_\mathcal{R} \mathcal{U}_\mathcal{R}^T + \mathcal{U}_\mathcal{N} \mathcal{U}_\mathcal{N}^T  & 0_{d\times\mathcal{K}} \\
    0_{\mathcal{K}\times d} &  \mathcal{U}_{\mathcal{R},i} \mathcal{U}_{\mathcal{R},i}^T  \\
   \end{bmatrix}
   \begin{bmatrix}
    \nabla_{\mathcal{T}_{in}^t} \mathcal{L}_i \\
    \nabla_{\mathcal{T}_i^t} \mathcal{L}_i  \\
   \end{bmatrix} \\
   &\simeq \sum_{i=1}^{\mathcal{K}} 
   \begin{bmatrix}
    \mathcal{U}_\mathcal{N} \mathcal{U}_\mathcal{N}^T  & 0_{d\times\mathcal{K}} \\
    0_{\mathcal{K}\times d} &  \mathcal{U}_{\mathcal{R},i} \mathcal{U}_{\mathcal{R},i}^T  \\
   \end{bmatrix}
   \begin{bmatrix}
    \nabla_{\mathcal{T}_{in}^t} \mathcal{L}_i \\
    \nabla_{\mathcal{T}_i^t} \mathcal{L}_i  \\
    \end{bmatrix} \\
    &= \sum_{i=1}^{\mathcal{K}} 
   \begin{bmatrix}
    (\mathcal{U}_\mathcal{N} \mathcal{U}_\mathcal{N}^T) \nabla_{\mathcal{T}_{in}^t} \mathcal{L}_i \\
    (\mathcal{U}_{\mathcal{R},i} \mathcal{U}_{\mathcal{R},i}^T) \nabla_{\mathcal{T}_i^t} \mathcal{L}_i  \\
    \end{bmatrix}
\end{align}

The total multi-task loss can be expressed as follows:
\begin{align}
    \mathcal{L}_i (\Theta_s^{t+1}, \mathcal{T}_{in}^{t+1}, \mathcal{T}_i^{t+1}) = 
    \mathcal{L}_i (\Theta_{in}^{t}, \mathcal{T}_s^{t}, \mathcal{T}_i^{t}) 
    + \nabla_{\Theta_s^t} \mathcal{L}_i (\Theta_s^{t}, \mathcal{T}_s^{t}, \mathcal{T}_i^{t})(\Theta_s^{t+1}-\Theta_s^t) \\
    + \nabla_{\mathcal{T}_{in}^t} \mathcal{L}_i (\Theta_s^{t}, \mathcal{T}_s^{t}, \mathcal{T}_i^{t})(\mathcal{T}_{in}^{t+1}-\mathcal{T}_{in}^{t}) \\
    + \nabla_{\mathcal{T}_i^t} \mathcal{L}_i (\Theta_s^{t}, \mathcal{T}_s^{t}, \mathcal{T}_i^{t})(\mathcal{T}_i^{t+1}-\mathcal{T}_i^t) \\
    = \mathcal{L}_i (\Theta_s^{t}, \mathcal{T}_{in}^{t}, \mathcal{T}_i^{t}) - \eta \nabla_{\Theta_s^t} \mathcal{L}_i (\Theta_s^{t}, \mathcal{T}_s^{t}, \mathcal{T}_i^{t}) \cdot \sum_{i=1}^{\mathcal{K}} \nabla_{\Theta_s^t} \mathcal{L}_i (\Theta_s^{t}, \mathcal{T}_{in}^{t}, \mathcal{T}_i^{t}) \\
    - \eta (\mathcal{U}_\mathcal{N} \mathcal{U}_\mathcal{N}^T) \nabla_{\mathcal{T}_{in}^{t}} \mathcal{L}_i (\Theta_s^{t}, \mathcal{T}_{in}^{t}, \mathcal{T}_i^{t}) \cdot \sum_{i=1}^{\mathcal{K}} (\mathcal{U}_\mathcal{N} \mathcal{U}_\mathcal{N}^T) \nabla_{\mathcal{T}_{in}^{t}} \mathcal{L}_i (\Theta_s^{t}, \mathcal{T}_{in}^{t}, \mathcal{T}_i^{t}) \label{eq:gc_null}\\ 
    - \eta (\mathcal{U}_{\mathcal{R},i} \mathcal{U}_{\mathcal{R},i}^T) \nabla_{\mathcal{T}_i^t} \mathcal{L}_i (\Theta_s^{t}, \mathcal{T}_{in}^{t}, \mathcal{T}_i^{t}) \cdot (\mathcal{U}_{\mathcal{R},i} \mathcal{U}_{\mathcal{R},i}^T) \nabla_{\mathcal{T}_i^t} \mathcal{L}_i (\Theta_s^{t}, \mathcal{T}_{in}^{t}, \mathcal{T}_i^{t}) \label{eq:token_expan}
\end{align}
The increase in multi-task loss caused by gradient conflicts in the null space (as described in \cref{eq:gc_null}) cannot be reduced since the shared token $\mathcal{T}_{in}^t$ is not a learnable parameter. Instead, task-specific tokens $\mathcal{T}_i^t$ can be added to mitigate the increase in multi-task loss due to null space gradient conflicts by optimizing the learnable parameters $\{\mathcal{T}_i\}_{i=1}^{\mathcal{K}}$ as described in \cref{eq:token_expan}.

\end{proof}

\end{document}